\documentclass[journal, 10pt]{IEEEtran}
\ifCLASSINFOpdf
\else
\fi

\usepackage{amsmath, amssymb, amsthm}
\usepackage{hyperref}
\usepackage{color}
\usepackage{amsfonts,fancyhdr,mdframed, bm}
\usepackage[utf8]{inputenc}
\usepackage{caption}
\usepackage{subcaption}
\usepackage{graphicx}
\usepackage{nicematrix}
\usepackage{cite}
\usepackage{array}
\usepackage{stfloats}
\usepackage{color}
\usepackage{booktabs}
\usepackage{multirow, makecell}
\usepackage{bbm}
\usepackage{makecell}
\usepackage{siunitx}
\usepackage{algorithm}
\usepackage{algpseudocode}

\newcommand{\negSpace}{-20pt}

\newtheorem*{problem*}{Problem}
\newtheorem{assumption}{Assumption}
\newtheorem{remark}{Remark}
\newtheorem*{example*}{Example}
\newtheorem{definition}{Definition}

\newtheorem{theorem}{Theorem}
\makeatletter
\NewCommandCopy\@@pmod\pmod
\DeclareRobustCommand{\pmod}{\@ifstar\@pmods\@@pmod}
\def\@pmods#1{\mkern4mu({\operator@font mod}\mkern 6mu#1)}
\makeatother

\DeclareMathOperator*{\argmin}{argmin}

\begin{document}

\title{Collision Probability Estimation for Optimization-based Vehicular Motion Planning}

\author{Leon Tolksdorf$^{1,2}$, Arturo Tejada$^{1, 3}$, Christian Birkner$^{2}$, and Nathan van de Wouw$^{1}$
\thanks{$^{1}$Department of Dynamics and Control, Eindhoven University of Technology, Eindhoven, The Netherlands, e-mail:
        {\tt\small \{l.t.tolksdorf, a.tejada.ruiz, n.v.d.wouw\}@tue.nl}}%
\thanks{$^{2}$CARISSMA Institute of Safety in Future Mobility, Technische Hochschule Ingolstadt, Ingolstadt, Germany, e-mail:
        {\tt\small \{leon.tolksdorf, christian.birkner\}@thi.de}}%
\thanks{$^{3}$TNO, Integrated Vehicle Safety, Helmond, The Netherlands, e-mail:
        {\tt\small arturo.tejadaruiz@tno.nl}}%
}

\maketitle

\begin{abstract}
Many motion planning algorithms for automated driving require estimating the probability of collision (POC) to account for uncertainties in the measurement and estimation of the motion of road users. Common POC estimation techniques often utilize sampling-based methods that suffer from computational inefficiency and a non-deterministic estimation, i.e., each estimation result for the same inputs is slightly different. In contrast, optimization-based motion planning algorithms require computationally efficient POC estimation, ideally using deterministic estimation, such that typical optimization algorithms for motion planning retain feasibility. Estimating the POC analytically, however, is challenging because it depends on understanding the collision conditions (e.g., vehicle's shape) and characterizing the uncertainty in motion prediction. In this paper, we propose an approach in which we estimate the POC between two vehicles by over-approximating their shapes by a multi-circular shape approximation. The position and heading of the predicted vehicle are modelled as random variables, contrasting with the literature, where the heading angle is often neglected. We guarantee that the provided POC is an over-approximation, which is essential in providing safety guarantees. For the particular case of Gaussian
uncertainty in the position and heading, we present a computationally efficient algorithm for computing the POC estimate. This algorithm is then used in a path-following stochastic model predictive controller (SMPC) for motion planning. With the proposed algorithm, the SMPC generates reproducible trajectories while the controller retains its feasibility in the presented test cases and demonstrates the ability to handle varying levels of uncertainty.
\end{abstract}

\begin{IEEEkeywords}
Collision Probability, Circular Approximations, Stochastic Model Predictive Control, Path-following
\end{IEEEkeywords}

\IEEEpeerreviewmaketitle

\section{Introduction}\label{sec_introduction}
Automated vehicles (AVs) usually operate with limited knowledge about their surroundings, which stems from limitations in sensing and estimation, e.g., range limitations or physical sight obstructions, noise in the sensed data, or missing information about the intentions of other road users. This limited knowledge, commonly referred to as uncertainty, is an important factor that should be incorporated in the design of an AV's motion planning and control algorithms as recommended in scientific literature (see, e.g., \cite{schwarting2018planning, McAllister.82017, hubmann2018automated, Tolksdorf_2025}) and safety standards \cite{ISO21448}. \\
The specifics of how to incorporate uncertainty in the design, however, depends on the particular algorithms and their implementation. Regarding motion planning and decision-making, uncertainty in the position and orientation of other actors and objects is understood as a challenge to ensure safety. Here, the probability of collision (POC) \cite{schwarting2017safe, goulet2022probabilistic, schreier2016integrated, althoff2009model} or related quantities, e.g., stochastic risk \cite{tolksdorf2023risk, hruschka2019uncertainty, nyberg2021risk}, are measures commonly used to ensure safety by constraining or minimizing them within a motion planning algorithm. \\
As mentioned, from an AV perspective, the positions and heading angles of other vehicles are uncertain. A common method to express this uncertainty is by treating the uncertain variables as random variables with specific associated probability density functions (PDFs). The POC for two arbitrarily shaped vehicles is then typically estimated with Monte Carlo sampling (MCS) \cite{lambert2008collision, schreier2016integrated}. That is, the POC is estimated by sampling the positions and orientations from their PDFs and checking whether these positions and orientations lead to a collision. The ratio of colliding to non-colliding samples then gives the POC estimate. We will refer to this method for POC estimation as MCS. Although applicable for any vehicle shape, MCS is computationally expensive as many samples are necessary for an accurate estimation; additionally, the POC estimate fluctuates around its true value for each sampling set. Hence, an inherent MCS drawback is that it is not guaranteed to \textit{not} under-approximate the POC for a finite amount of samples, which is generally undesirable in safety-critical applications.\\
To improve upon the MCS approach, various analytic methods for POC estimation have been proposed. For instance, using occupancy grids. That is, the motion planning space in $\mathbb{R}^2$ is discretized into cells, and the probability of occupation by both actors of a cell is calculated \cite{althoff2011comparison}, \cite{candela2023risk}. In \cite{althoff2011comparison}, the POC is approximated with Markov chains and compared with MCS, where it is concluded that MCS is the better approximation in accuracy and computational efficiency. 
Aside from occupancy grids, \cite{du2011probabilistic} chooses to model a robot and an obstacle as a circle each and formulates a joint Gaussian PDF of each vehicle's position to approximate the POC analytically. The authors of \cite{philipp2019analytic} consider rectangular-shaped vehicles, of which the position of one is known, whereas the position of the other vehicle is Gaussian distributed. From the distributions, the PDFs are numerically integrated to retrieve the POC; however, both vehicles' heading angles are assumed to be deterministic. The same approach, i.e., using rectangles with deterministic heading angles, is utilized by \cite{altendorfer2021new} to formulate the POC over a time interval, where an analytic solution is derived for an approximation of the POC. An analytic approximation of the POC is also proposed by \cite{patil2012estimating}; here, the vehicle's shape is assumed to be point-like and the uncertainty is Gaussian. \\
Another shape approximation in motion planning is the multi-circular approximation, where a rectangle is covered with overlapping circles \cite{ziegler2010fast, ziegler2014trajectory, schwarting2017safe, gutjahr2016lateral, manzinger2020using, werling2012optimal, micheli2023nmpc}. Computing the POC analytically for a multi-circular AV approximation has been presented in \cite{tolksdorf2024}, where it is also shown that the proposed analytic method significantly outperforms MCS in computational speed while a slight over-approximation of the true POC is guaranteed. In spite of the benefits, the proposed method only considers the AV to be approximated by multi-circles and all vehicles with whom it is potentially colliding by a single circle, avoiding the need to consider the heading angle, yet undermining the accuracy of this approximation. In \cite{mustafa2024racp}, the vehicle's shapes are each approximated by a two-circle covering, though the heading angle is assumed to be deterministic. Evidently, for estimating the POC for two vehicles of which the position \textit{and} orientation of one vehicle is uncertain, the main challenge is to guarantee not to underestimate the true POC while generating a not-too-conservative overestimation.\\ 
The motion plan of an AV can be generated with any of several techniques (see, e.g., \cite{Gonzalez.2016, Claussmann.2020}). Here, we focus on stochastic model predictive control (SMPC) as an optimization-based technique. SMPC allows for the flexible integration of the POC either as a constraint or within a cost function to be minimized. Further, path-following SMPC (see \cite{Faulwasser.2016} for details) has been shown to generate smooth trajectories while successfully adapting the motion plan to a given level of uncertainty \cite{tolksdorf2023risk}. The drawbacks of SMPC, however, are its high computational burden and the general lack of guarantee that it is recursively feasible. Further, the evaluation of probabilistic constraints (e.g., constraints based on the POC) is known to be challenging, as, in many cases, sampling-based estimation techniques (online MCS) are computationally heavy or analytic methods are strongly over-approximative, leading to overly conservative behavior \cite{mesbah2016stochastic}. Suppose, for instance, the POC is estimated with MCS within an optimization-based motion planner. If the optimizer converges towards the boundary of the POC constraint, a previously feasible point might suddenly become infeasible due to an unfavorable random draw \cite{homem2014monte}. Therefore, results, i.e., planned trajectories, are non-reproducible as the fluctuations in POC estimation are random, and the optimizer might struggle to converge as a consequence. \\
The main contributions of this paper are as follows: we address the aforementioned challenges for computing the POC by presenting a novel method based on multi-circular shape approximations for all actors to estimate the POC, particularly for motion-planning applications. This method overcomes the drawbacks of MCS, i.e., its high computational load, the lack of a guarantee to over-approximate the POC, and POC fluctuations causing the planning optimizers to fail. We present the necessary geometric conditions for multi-circle collisions to guarantee that the method provides an over-approximation, while practical case studies demonstrate that the results are not overly conservative. Furthermore, we present an algorithm to compute the POC for a Gaussian positional uncertainty and a wrapped Gaussian distribution for the heading angle. On the basis of a representative numerical case study, we discuss the trade-off between computational efficiency and over-approximation accuracy. Lastly, we present an SMPC scheme for motion planning leveraging the advantages of the proposed POC estimation method: computing smooth, reproducible trajectories while the controller retains its recursive feasibility in the tested scenarios and reacts conservatively when the uncertainty grows. \\
The remainder of this article is organized as follows. Section \ref{sec_problem_statment} provides preliminary notation and presents the problem statement of computing the POC using multi-circle shape approximations. Further, Section \ref{sec_collision_events} bounds the set of all collision configurations. Section \ref{sec_combination_of_proabilities} guarantees the over-approximation of the POC. Section \ref{sec_gaussian_POC_algo} presents a method to compute the POC for Gaussian uncertainty, and the SMPC motion planning scheme is presented in Section \ref{sec_SMPC}. The POC estimation for Gaussian uncertainties and the SMPC are then simulated in representative case studies in Section \ref{sec_simulation_example}, and the results are discussed in Section \ref{sec_discission}. Finally, conclusions and recommendations for future work are given in Section \ref{sec_conclusions_and_future_work}.

\section{Problem Statement: POC Estimation}\label{sec_problem_statment}
\begin{figure}[t!]
\begin{center}
\includegraphics[width=8.4cm]{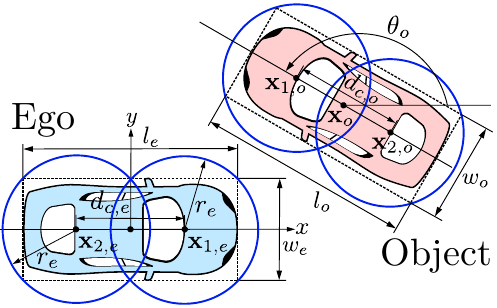}  
\caption{Schematic of the problem statement. }
\label{fig_problem_statement}
\end{center}
\end{figure}
Consider an arbitrary traffic scene consisting of an automated vehicle and another vehicle, referred to as ego and object vehicles. As the ego performs the POC estimation, all variables are defined with respect to the ego's geometric center. By assumption, the object  does not communicate any information to the ego (e.g., its position or intentions). We characterize each vehicle by a configuration composed of a position and orientation. Let any position in $\mathbb{R}^2$ be denoted by $\mathbf{x} := (x, y)$ and any orientation be denoted by $\theta \in \mathbb{R}_{\geq 0}$. Together, they construct the configuration space $\mathcal{C} := \mathbb{R}^2 \times \mathbb{R}_{\geq 0}$, giving the set of all possible vehicle configurations. For later calculations, we will restrict $\theta$ to positive values for convenience. Whenever confusion can arise, we identify variables associated with the ego and the object, respectively, with $e$ and $o$ subscripts. Let $\mathbf{y}_o := (\mathbf{x}_o, \theta_o) \in \mathcal{C}$ characterize the object's configuration. That is, $\mathbf{x}_o$ denotes the position of its geometrical center and $\theta_o$ denotes its heading angle measured with respect to the ego's $x$-axis (see Figure \ref{fig_problem_statement}).
A set of integers $\{a, a + 1, ..., b \}$, with $ a < b$, is denoted by $\mathbb{Z}_a^b$. We denote a configuration $\mathbf{y}$ at discrete-time instant $k$ as $\mathbf{y}_{k}$ and a predicted configuration at discrete-time instant $n \geq k$ given information available at time $k$ as $\mathbf{y}_{n|k}$. The same notation is applied for states, inputs, and constraints.  All random variables are defined over the same probability space $(\Omega, \mathcal{F}, \mathbb{P})$, where $\Omega$ is the sample space, $\mathcal{F}$ is a sigma-algebra over $\Omega$, and $\mathbb{P}$ is a probability measure over $\mathcal{F}$. The sets in $\mathcal{F}$, also known as events, are denoted by $\mathcal{A}_m$ with $m \in \mathbb{N}$. In the sequel, the following definitions will be needed. 

\textbf{\begin{definition}\label{def_prob_measure}(Probability Measure)
A mapping $\mathbb{P}: \mathcal{F} \rightarrow [0, 1]$ is a probability measure if:
    \begin{itemize}
        \item[(i)] $0 \leq \mathbb{P}\{\mathcal{A}_m \} \leq 1$, for any $\mathcal{A}_m \in \mathcal{F}$,
        \item[(ii)] $\mathbb{P}\{\Omega\} = 1$ and $\mathbb{P}\{\emptyset\} = 0$,
        \item[(iii)] $\mathbb{P}\{\bigcup_{m =1}^{\infty} \mathcal{A}_m\} = \sum_{m =1}^{\infty}\mathbb{P}\{ \mathcal{A}_m\}$, if $\mathcal{A}_m$ are pairwise disjoint (i.e., $\mathcal{A}_m \cap \mathcal{A}_j = \emptyset$, for $m\!\neq\!j$ and $\mathcal{A}_m,\mathcal{A}_j \in \mathcal{F}$).
    \end{itemize}
\end{definition}}
\begin{definition}\label{def_random_vector}(Random Element)
    Let $(\Omega, \mathcal{F})$ and $(\tilde{\Omega}, \tilde{\mathcal{F}})$ be measurable spaces. A map
    $\mathbf{x} : \Omega \rightarrow \tilde{\Omega}$ is called a random element if $\mathbf{x}^{-1}[\tilde{\mathcal{A}}_m] \in \mathcal{F}$ for all $\tilde{\mathcal{A}}_m \in \tilde{\mathcal{F}}$.
\end{definition}
\subsection{Collision Probability for Arbitrary Shapes} \label{sec_poc_arbritrary_shapes}
\begin{figure}[t!]
\begin{center}
\includegraphics[width=8.8cm]{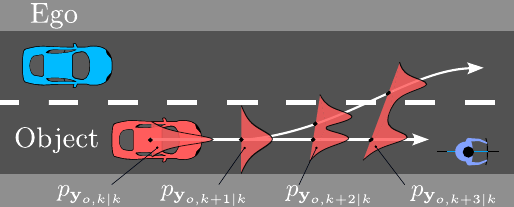}  
\caption{The ego estimates and predicts the uncertainty about the object's current and future configurations. Note, the PDFs $p_{\mathbf{y}_{o,n|k}}$ may have a different functional form for each time-step, thus accommodating different possible object behaviors.}
\label{fig_prediction}
\end{center}
\end{figure}
For motion planning purposes, the ego vehicle (referred to as ego from now on) must estimate the POC with the object vehicle (referred to as object from now on) at all times along its motion plan. Typically, the ego is equipped with a sensor suit that perceives the object, from which the information is processed by algorithms estimating its current configuration and predicting the future configurations of the object. Clearly, the last step, i.e., predicting future configurations of the object, introduces the greatest amount of uncertainty, especially given an already uncertain current configuration estimate. Hence, we abstract the process of measurement, estimation, and prediction by letting the object's configuration be uncertain at any point in time. That is, at a current time instance $k$, or predicted time instance $n \geq k$, a PDF can characterize the uncertainty in the configuration of the object. We note that this concept can be applied to many state-of-the-art road user prediction models, where the PDF can be any distribution over the configuration space $\mathcal{C}$, even accommodating very different object behaviors (see the example illustrated in Figure~\ref{fig_prediction}). Consequently, we adopt the following assumption about the object.
\begin{assumption}(Object information)\label{ass:object_info}
    The ego estimates the kinematic variables $\mathbf{y}_{o,n|k}$ at discrete-time instance $n \geq k$, given information available at time $k$, together with their associated probability density function $p_{\mathbf{y}_{o},n|k}$, for all instances $n$ within a prediction horizon $\mathbb{Z}_k^{k+N_P}$. All components in $p_{\mathbf{y}_{o},n|k}$ are mutually independent.
\end{assumption}
A collision occurs if the spaces occupied (i.e., the footprints) by the ego and the object intersect, regardless of how both got there. Thus, we note that the velocity is not considered to determine the POC, as the configuration is evidently a consequence of applying a velocity over time\footnote{Consider, e.g., that the object configurations are predicted by a dynamic model, propagating estimated object inputs (steering and acceleration) to object configurations. Here, the inputs are random variables, causing the predicted configurations also to be random variables. However, only the configurations are required to determine the occurrence of collisions, and, hence, the POC estimation \textit{method} is independent of the object's inputs. Nonetheless, the object's inputs are implicitly considered by the prediction model, propagating those into (future) configurations; thus, the estimation \textit{result} is dependent on the object's inputs.}. Given the fact that the method of POC estimation is the same for all time instants, we omit time indexing in the sequel for notational simplicity. Without loss of generality, we assume that both footprints are rectangles over $\mathbb{R}^2$ denoted, respectively, by $\mathcal{S}_e$ and $\mathcal{S}_o(\mathbf{y}_{o})$ (see dashed boxes in Figure \ref{fig_problem_statement}). As all measurements are with respect to the ego, the POC is a function of the statistics of $\mathbf{y}_{o}$ only. More specifically, if $\tilde{\mathcal{A}}_{rec} := \{\mathbf{y} \in \mathcal{C} \mid \mathcal{S}_e \cap\mathcal{S}_o(\mathbf{y}) \neq \emptyset\}$ denotes the object configurations that lead to a collision, then the POC is given by
\begin{equation}\label{eq_general_POC}
\textup{POC}\triangleq \mathbb{P}\{\mathbf{y}_{o} \in\tilde{\mathcal{A}}_{rec}\} = \int\displaylimits_{ \tilde{\mathcal{A}}_{rec}} p_{\mathbf{y}_o}(\mathbf{y}) \text{d} \mathbf{y}.
\end{equation}

\subsection{Multi-Circle Footprint Approximations}\label{sec_problem_multi-circle}
Although both footprints $\mathcal{S}_e$ and $\mathcal{S}_o(\mathbf{y}_o)$ are considered rectangular, we propose to use a multi-circle footprint cover. In the sequel, $\mathcal{B}[\mathbf{x}_c;r]$ denotes a closed circle with center $\mathbf{x}_c = (x_c, y_c)$ and radius $r \in \mathbb{R}_{>0}$. That is, 
\begin{equation}\label{eq_circle}
\mathcal{B}[\mathbf{x}_c;r] := \{ \mathbf{x} \in \mathbb{R}^2 \mid \| \mathbf{x}_c - \mathbf{x}\| \leq r \}.
\end{equation}
Suppose that the smallest rectangular approximation of a vehicle's footprint has length $l$ and width $w$, $l \geq w$. Further, suppose a vehicle's footprint is covered by $N_c$ overlapping, closed circles with equal radii $r$ and centers $\mathbf{x}_i(\mathbf{y})$, $i\in\mathbb{Z}_1^{N_c}$, where $\mathbf{y}$ denotes the vehicle configuration. That is, 
\begin{equation}\label{eq_shape}
\mathcal{S}(\mathbf{y}) \subset \mathcal{BC}(\mathbf{y}) := \bigcup_{i=1}^{N_c}\mathcal{B}[\mathbf{x}_i(\mathbf{y});r].
\end{equation}
It can be shown that the smallest radius $r$ needed to fully cover the rectangle, when all circles are placed equidistantly along the central longitudinal axis of the rectangle (with distance $d_c$), is given by
\begin{equation}\label{eq_placing}
    r\!=\! \sqrt{\left(\!\frac{l}{2N_c }\!\right)^2\! +\! \frac{w^2}{4}}, \; d_c\!:= \!2\sqrt{r^2\!-\!\frac{w^2}{4}}\!=\!\|\mathbf{x}_{i+1}(\mathbf{y})\!-\!\mathbf{x}_i(\mathbf{y})\|
\end{equation}
with $ i\in \mathbb{Z}_1^{N_c -1}$. 
The problem setup is depicted in Figure \ref{fig_problem_statement} for a two-circle ($N_c = 2$) vehicle footprint covering. \\
\begin{problem*}
Derive an estimate of the POC in (\ref{eq_general_POC}) by using the ego and object footprint covering (\ref{eq_circle}) - (\ref{eq_placing}).
\end{problem*}
The problem is approached by first identifying the intersection conditions (i.e., conditions when multi-circular coverings of two vehicles collide) for multi-circular coverings, yielding a set of all possible intersection configurations\footnote{As will be shown in Theorem \ref{theo_over_approx}, the collision configurations are a subset of the intersection configurations presented in Section \ref{sec_collision_events}, which is why our POC estimate is guaranteed to be an over-approximation of the actual POC.}, and subsequently, integrating the PDF $p_{\mathbf{y}_o}$ over this set to obtain the POC estimate.

\section{Intersection Conditions for Multi-Circle-to-Multi-Circle Collisions}\label{sec_collision_events}
\begin{figure*}[t!]
\begin{center}
\includegraphics{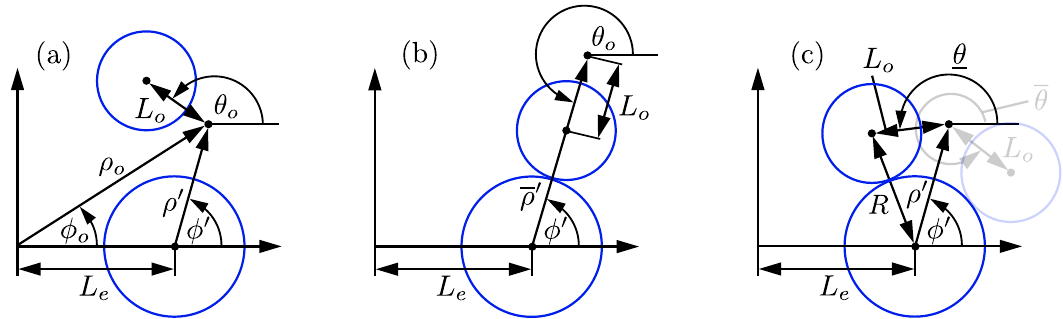}  
\caption{(a) Problem setup for some $(\phi_o, \rho_o, \theta_o)$ and shifting the polar frame by distance $L_e$ to $\rho', \phi'$. (b) Upper-bounding $\rho'$. (c) Bounding $\theta_o$ for a given $\rho'$ and $\phi'$.}
\label{fig_offset_circle}
\end{center}
\end{figure*}
To estimate the POC from the given PDF $p_{\mathbf{y}_o}$, one is tasked with defining a set $\tilde{\mathcal{A}} \subset \mathcal{C}$, for which conditions for a collision are satisfied, see (\ref{eq_general_POC}). Clearly, some object configurations can be excluded as they are too distant from the ego for a collision to occur, irrespective of the object's heading angle. In fact, let us define the relative distance of the object's geometric center to the ego's geometric center as $\rho_o = \sqrt{x_o^2 + y_o^2}$. Further, denote $R := r_e + r_o$ as the joint radius, with $r_e$ and $r_o$, respectively, the radii of the circles covering the ego and the object (given by (\ref{eq_placing})). Herewith, we can upper-bound the relative distance for which a collision can occur to a maximum collision distance as 
\begin{equation}\label{radial_bound}
    \overline{\rho} = R + \frac{d_{c,o}}{2}\big(N_{c,o} - 1\big) + \frac{d_{c,e}}{2}\big(N_{c,e} -1\big),
\end{equation}
where $N_{c,e}, N_{c,o}$ denote the number of circles covering the ego and object and $d_{c,e}, d_{c,o}$ are defined as in (\ref{eq_placing}) for the ego and object, respectively. Note that for any relative distance $\rho_o > \overline{\rho}$, a collision cannot occur. Using this inequality and the radial bounding (\ref{radial_bound}) intuitively motivates the usage of a polar coordinate frame, where we define a change of coordinates as 
\begin{equation}\label{eq:polar_transform}
\begin{split}
   & CT: \mathbb{R} \times [0, 2\pi) \rightarrow \mathbb{R}^2,\\
   & (\rho, \phi) \mapsto (x, y) = (\rho \cos{\phi}, \rho \sin{\phi}).
\end{split}
\end{equation} 
Using (\ref{eq:polar_transform}) and the two-argument arctangent function `$\text{atan2}$', the random vector 
\begin{equation}\label{eq:polar_y}
    \mathbf{y}_{o,p} := (\phi_o, \rho_o, \theta_o) = (\text{atan2}(y_o, x_o), \sqrt{x_o^2 + y_o^2},  \theta_o)
\end{equation}
is derived, where the subscript $p$ denotes that $\mathbf{y}_o$ is expressed in polar coordinates\footnote{Note that $\mathbf{y}_{o,p} = (\mathbf{x}_{o,p}, \theta_o)$; where $\mathbf{x}_{o,p} = CT^{-1}(\mathbf{x}_o)$, hence, $\mathbf{x}_{o,p} = (\phi_o, \theta_o)$ is no longer Gaussian, and $\phi_o$ and $\rho_o$ are in general not independent (although $\mathbf{x}_{o,p}$ and $\theta_o$ remain independent).}.
As such, we can restrict the position coordinates of object positions leading to a collision to $(\phi_o, \rho_o) \in [0, 2\pi) \times [0, \overline{\rho}]$, since outside this set, no collisions can occur, irrespective of the object's heading. Within that set, however, several conditions for collisions can be found. To examine those conditions, we first study the base case of two \textit{offset}-circles intersecting. The more general collision conditions are derived later based on this base case.
\subsection{Base Case: Intersection Conditions for Offset-Circles}
We denote circles as \textit{offset} when their centers are not placed in the geometric center of their associated vehicles but offset by a distance $L_e$ and $L_o$. The base case is depicted in Figure \ref{fig_offset_circle}(a). As the ego circle is offset by the distance $L_e$ along the ego's $x$-axis, we transform $(\phi_o, \rho_o)$ to the offset center. The transformed polar coordinates are obtained as 
\begin{equation}\label{eq:polar_shift}
   \phi' = \text{atan2}(y_o, (x_o - L_e)) \text{ and }  \rho' = \sqrt{(x_o -L_e)^2 + y_o^2}.
\end{equation}
Note that $x_o, y_o$ are related to $\phi_o$ and $\rho_o$ by (\ref{eq:polar_transform}). Evidently, for an intersection to occur in the shifted polar coordinate frame, a new radial upper-bound $\overline{\rho}' := R + L_o$ on $\rho'$ is found, see Figure \ref{fig_offset_circle}(b). 
Whenever the intersection depends on the heading angle $\theta_o$, the lower-bound $\underline{\theta}$ and upper-bound $\overline{\theta}$ characterize the first and last contact point between both circles (see Figure \ref{fig_offset_circle} (c)), which can be determined geometrically as
\begin{equation}\label{eq_circle_v_offset_circle_heading_bounds}
\begin{split}
    &\overline{\theta}(\phi', \rho') = \phi' + \pi + \arccos{\left(\frac{L_o^2 + \rho'^2 - R^2}{2L_o\rho'}\right)},\\
    &\underline{\theta}(\phi', \rho') = \phi' + \pi - \arccos{\left(\frac{L_o^2 + \rho'^2 - R^2}{2L_o\rho'}\right)}.
\end{split}
\end{equation}
The full derivation of (\ref{eq_circle_v_offset_circle_heading_bounds}) is provided in Appendix \ref{appendixA}.
The conditions for offset-circle intersections are summarized in Table \ref{tab_collision_conditions_circle_v_circle}. Here, the column \textit{Intersection angle interval} gives the range of heading angles that lead to an  intersection under the condition that the \textit{Radial condition} is satisfied. An empty interval (i.e., $\emptyset$) indicates that the circles cannot intersect for any angle $\theta_o \in [0, 2\pi)$. 
\begin{table}[htb]
\begin{center}
\caption{Intersection Conditions, Base case.}\label{tab_collision_conditions_circle_v_circle}
\begin{NiceTabular}{lll}
\toprule
\textbf{Geometry}&  \textbf{Radial Condition} & \textbf{Intersection Angle Interval}\\\midrule
\multirow{2}{*}{$ L_o > R$} &  $\rho' \geq L_o - R$& $[\underline{\theta}(\phi', \rho'), \overline{\theta}(\phi', \rho')]$ \\ \cmidrule(l){2-3}
 &  $\rho' < L_o -R $  &$\{ \emptyset \}$ \\
 \midrule
 \multirow{2}{*}{$ R \geq L_o$} &   $\rho' > R-L_o $&$[\underline{\theta} (\phi', \rho'), \overline{\theta} (\phi', \rho')]$ \\\cmidrule(l){2-3}
 &  $\rho' \leq R - L_o $ &$[0, 2\pi)$ \\
\bottomrule
\end{NiceTabular}
\end{center}
\end{table}
\begin{remark}
    Note that for a given $\rho\in [0, \overline{\rho}]$, the transformed radial distance $\rho'$ is derived from (\ref{eq:polar_shift}) with (\ref{eq:polar_transform}) and used to obtain the intersection angle intervals (\ref{eq_circle_v_offset_circle_heading_bounds}) under the conditions from Table \ref{tab_collision_conditions_circle_v_circle}. Still, the radial domain of possible collisions is $[0, \overline{\rho}]$, even though there can be a $\rho \in [0, \overline{\rho}]$, for which an intersection among two specific circles does not occur, as $\overline{\rho}$ is a bounding for a intersection between for any pair of ego and object circles. For such a $\rho$, the corresponding intersection angle interval will be $\{\emptyset\}$ (see Table \ref{tab_collision_conditions_circle_v_circle}, second row). For example, suppose the object's position is in the ego's circle center, i.e., $\rho' = 0$. Yet, if the object's circle is offset by a distance $L_o > R$, than the ego's circle cannot intersect with the offset object's circle, although the object's center is in the ego's circle center. Contrarily, if $R \geq L_o$, then for $\rho' = 0$, both circles will intersect, irrespective of the object's heading angle (see Table \ref{tab_collision_conditions_circle_v_circle}, fourth row). 
\end{remark}
\subsection{General Case: Intersection Conditions for Multi-Circles}
Suppose the ego's and object's footprints are covered by $N_{c,e}$ and $N_{c,o}$ circles, respectively. Then, each ego circle can intersect with one or more object circles for the same object configuration $\mathbf{y}_o$. Therefore, one cannot add the individual POCs for each offset-circle-to-offset-circle intersection, as the intersection angle intervals are not disjoint. We tackle this problem by taking the union of all intersection angle intervals and then integrating over the resultant set, which generally can be a collection of many disjoint intervals. Denote $\mathbb{I}^{\theta_o}_{l,q}(\phi, \rho)$ the intersection angle interval for $\theta_o$ of the $l$-th ego offset circle intersecting with the $q$-th object offset circle, retrieved from Table \ref{tab_collision_conditions_circle_v_circle} (third column). Then, for a specific $(\phi, \rho) \in [0, 2\pi) \times [0, \overline{\rho}]$, the region of integration of the PDF for the object's heading angle is given by 
\begin{equation}\label{eq:collision_set}
    \tilde{\mathcal{A}}_{\theta_o}(\phi, \rho) = \bigcup_{l = 1}^{N_{c,e}}\bigcup_{q = 1}^{N_{c,o}} \mathbb{I}^{\theta_o}_{l,q}(\phi, \rho).
\end{equation}
With the union of all intersection angle intervals ensuring that all intervals within $\tilde{\mathcal{A}}_{\theta_o}(\phi, \rho)$ are disjoint, we find the region of integration, i.e., the set of all object configurations leading to a collision with the ego vehicle as:
\begin{equation}\label{eq:Acir}
    \tilde{\mathcal{A}}_{cir}\!=\! \{(\phi, \rho, \theta) \in \mathcal{C}_p\! \mid \! \phi \in [0, 2\pi), \rho \in [0, \overline{\rho}], \theta \in \tilde{\mathcal{A}}_{\theta_o}(\phi, \rho) \},
\end{equation}
where $\mathcal{C}_p$ is the polar transformation of the configuration space, i.e., $\mathcal{C}_p := CT^{-1}(\mathbb{R}^2) \times \mathbb{R}_{\geq 0}$.
\begin{remark}\label{rem:intervals}
Note that the effort of computing all $N_{c,e} \times N_{c,o}$ intersection angle intervals in (\ref{eq:collision_set}) can be reduced, as symmetries from pairs of circles can be exploited. For example, suppose $[\underline{\theta}_1, \overline{\theta}_1]$ is the intersection angle interval for an object circle offset by $L_o$ along the object's longitudinal axis. Then, the intersection angle interval for the object circle offset by $-L_o$ along the object's longitudinal axis is given by $[\underline{\theta}_1 + \pi, \overline{\theta}_1+ \pi]$. Furthermore, if for a specific $(\phi, \rho)$ one intersection angle interval is full range, i.e., $[0, 2\pi)$, any other interval can be discarded, as the PDF of $\theta_o$ will then be integrated over its entire domain and, hence, the integral over $\theta$ (when the integral in (\ref{eq_general_POC}) is integrated by parts over each random variable) will evaluate to one, regardless of all other intervals.  
\end{remark}
\section{Probability of Multi-Circle-to-Multi-Circle Collision}\label{sec_combination_of_proabilities}
Note that (\ref{eq:collision_set}) may be a union of disjoint intervals. Therefore, we denote $\tilde{\mathbb{I}}_{i}^{\theta_o}(\phi, \rho)$ as one of the $N_I$ disjoint sub-intervals of $\tilde{\mathcal{A}}_{\theta_o}(\phi, \rho)$, such that $\bigcup_{i = 1}^{N_I}\tilde{\mathbb{I}}_{i}^{\theta_o}(\phi, \rho) = \tilde{\mathcal{A}}_{\theta_o}(\phi, \rho)$ and $\tilde{\mathbb{I}}^{\theta_o}_{i}(\phi, \rho) \cap \tilde{\mathbb{I}}^{\theta_o}_{j}(\phi, \rho) = \emptyset$ for all $i \neq j$ and $i, j \in \{1, 2, ..., N_I \}$. Given Assumption \ref{ass:object_info}, the polar coordinate transformation (\ref{eq:polar_transform}), (\ref{eq:polar_y}) and the set of collision configurations (\ref{eq:Acir}), an estimate of the POC, which is guaranteed to be an upperbound for the real POC, is given by 
\begin{equation}\label{eq:poc}
\begin{split}
    &\mathbb{P}\{\mathbf{y}_{o,p} \in \tilde{\mathcal{A}}_{cir}\} = \\
    &\int_0^{2\pi}\int_0^{\overline{\rho}}p_{\phi_o, \rho_o}(\phi, \rho)\left[ \sum_{i=1}^{N_{I}} \int_{\tilde{\mathbb{I}}^{\theta_o}_{i}(\phi, \rho)}p_{\theta_o}(\theta)\text{d}\theta \right]\text{d}\rho \text{d}\phi.
    \end{split}
\end{equation}
The theorem hereafter expresses that the result of (\ref{eq:poc}) indeed is an over-approximation of the true POC. 
\begin{theorem}\label{theo_over_approx}(Over-approximate POC)
    Consider an ego and an object vehicle, the footprints of which are both over-approximated by multiple circles positioned according to (\ref{eq_placing}). The POC estimate in (\ref{eq:poc}) is an over-approximate of the POC as defined in (\ref{eq_general_POC}), i.e., $\mathbb{P}\{ \mathbf{y}_o \in \tilde{\mathcal{A}}_{rec} \} \leq \mathbb{P}\{ \mathbf{y}_{o,p} \in \tilde{\mathcal{A}}_{cir} \}$.
\end{theorem}
\begin{proof}
    By construction and (\ref{eq_shape}), $\tilde{\mathcal{A}}_{cir} = \{ \mathbf{y}_{o,p} \in \mathcal{C}_p \mid \mathcal{BC}_e(\mathbf{0}) \cap \mathcal{BC}_o(CT[\mathbf{y}_{o,p}]) \neq \emptyset\}$, where, with a slight abuse of notation, we define: $CT[\mathbf{y}_{o,p}] = (CT(\mathbf{x}_{o,p}), \theta_o)$, with $\mathbf{y}_{o,p} = (\mathbf{x}_{o,p}, \theta_o)$. \\
    It follows from this fact, the definition of $\tilde{\mathcal{A}}_{rec}$, (\ref{eq_shape}), and the invertibility of (\ref{eq:polar_transform}), that if $\mathbf{y}_o \in \tilde{\mathcal{A}}_{rec}$, then $\mathbf{y}_{o,p} = CT^{-1}[\mathbf{y}_o] = (CT^{-1}(\mathbf{x}_o), \theta_o) \in \tilde{\mathcal{A}}_{cir}$. Let $CT^{-1}[\tilde{\mathcal{A}}_{rec}] := \{ \mathbf{y}_{o,p} \in \mathcal{C}_p \mid \mathbf{y} = CT[\mathbf{y}_{o,p}] \in \tilde{\mathcal{A}}_{rec}\}$. Clearly, $CT^{-1}[\tilde{\mathcal{A}}_{rec}] \subset \tilde{\mathcal{A}}_{cir}$. Further, since $CT$ is a continuous (and, hence, measurable) map, then $CT^{-1}[\tilde{\mathcal{A}}_{rec}]$ is a measurable set. From this fact, it follows that $\mathbb{P}\{ \mathbf{y}_{o,p} \in \tilde{\mathcal{A}}_{cir} \} \geq \mathbb{P}\{ \mathbf{y}_{o,p} \in CT^{-1}[\tilde{\mathcal{A}}_{rec}] \} = \mathbb{P}\{\mathbf{y}_o \in \tilde{\mathcal{A}}_{rec} \}$. This completes the proof.
\end{proof}
\begin{remark}
Note that to obtain (\ref{eq:poc}), the assumption of independence of the random variables $x, y, \theta$ in Assumption \ref{ass:object_info} has been exploited, see, e.g., \cite{billingsley1995}, p. 262 - 264, for a comprehensive definition of independence. That is, the independence assumption allows to separate the PDF as $p_{\mathbf{y}_o}(x, y, \theta) = p_{x_o}(x)p_{y_o}(y)p_{\theta_o}(\theta)$. Let $h(x, y) = \sqrt{x^2 + y^2}$ and $g(x,y) = \text{atan2}(y, x)$, then $(\phi, \rho) = (h(x, y), g(x,y))$ is a measurable random vector as $h, g$ are measurable; however, $\phi, \rho$ are no longer independent of each other. Evidently, we can rearrange the PDFs as $p_{\mathbf{y}_o}(x, y, \theta) = p_{\phi_o, \rho_o}(\phi, \rho)p_{\theta_o}(\theta)$, which is the expression found in (\ref{eq:poc}).
\end{remark}
\begin{remark} 
While (\ref{eq:poc}) holds for many types of PDFs, it may also hold for other shapes, e.g., rectangles. However, with (\ref{eq:poc}) being shape-independent, the intersection conditions, expressed throughout Section \ref{sec_collision_events}, are not shape-independent. Here, circles are a particular effective shape approximation, as the intersection angle intervals can be straightforwardly computed from the analytic expression (\ref{eq_circle_v_offset_circle_heading_bounds}), which may not be possible for other shapes.
\end{remark}
\section{POC Algorithm for Gaussian Uncertainties}\label{sec_gaussian_POC_algo}
Theorem~\ref{theo_over_approx} in the previous section guarantees the over-approxinmation of the POC and it is fully generic in the sense that it can be used for a wide class of PDFs. In this section, we present an exemplary derivation of the POC over-approximation for Gaussian uncertainties given their prevalence in motion planning applications (see, e.g.,~\cite{dixit2019trajectory, volz2015stochastic, patil2012estimating, ploeg2022long, altendorfer2021new, du2011probabilistic}) and accompany that with an estimation algorithm. Note that while Gaussians cannot express different object behaviors, e.g., the object staying in lane or lane changing, as in Figure~\ref{fig_prediction}, different behaviors can still be accounted for by multiplying the Gaussian POC estimate with the likelihood of a specific behavior.\\
\subsection{POC Estimation for Gaussian Uncertainties}
Although the object's heading angle $\theta_o$ is defined on $\mathbb{R}_{\geq 0}$, to the ego it appears periodic on the interval $[0, 2\pi)$. Namely, from the ego's perspective, there is no difference if $2\pi$ is added to the value of $\theta_o$. The ego's perspective, however, inherently renders $\theta_o$ non-Gaussian as $\theta_o$ is bounded to $[0, 2\pi)$. The motivation for modeling  $\theta_o$ to be defined on $\mathbb{R}_{\geq 0}$ is due to convenience, as no further checks are needed when summing angles as in  (\ref{eq_circle_v_offset_circle_heading_bounds}), or exploiting symmetries as outlined in Remark \ref{rem:intervals}. To account for the periodicity from the ego's perspective, we use a wrapped normal distribution to assign a Gaussian-like PDF to the heading angle. Further, with $\mu_{\theta_o}, \sigma_{\theta_o}^2$ being the mean and variance of $\theta_o$, we let the wrapped-Gaussian distribution for the object's heading angle be given by 
\begin{equation}\label{eq_inf_wrapped_gauss}
    p_{\theta_o}(\theta) = \frac{1}{\sqrt{2\pi} \sigma_{\theta_o}}\sum_{\beta = -\infty}^{\infty}\exp{\left[ -\frac{(\theta + 2\pi \beta -\mu_{\theta_o})^2}{2\sigma_{\theta_o}^2} \right]}.
\end{equation}
Given the polar coordinate transformation (\ref{eq:polar_transform}), we can obtain the bivariate-Gaussian for the object's position $\mathbf{x}_o$ as a function of the polar coordinates $\phi$ and $\rho$ as follows:
\begin{equation}\label{eq_polar}
\begin{split}
    & p_{\phi_o, \rho_o}(\phi, \rho) = \frac{\rho}{2 \pi \sigma_{x_o}\sigma_{y_o}} \text{exp} \Bigg[ \frac{(\rho \cos{(\phi)} - \mu_{x_o})^2}{-2\sigma_{x_o}^2} \\
    & -\frac{(\rho \sin{(\phi)} -  \mu_{y_o})^2}{2\sigma_{y_o}^2}  \Bigg],
\end{split}
\end{equation} 
where $\mu_{x_o}, \mu_{y_o}, \sigma^2_{x_o}, \sigma^2_{y_o}$ are the means and variances of $x_o, y_o$. \\
As shown in \cite{kurz_wrapped_gauss}, the infinite sum in (\ref{eq_inf_wrapped_gauss}) can be tightly approximated by truncating it within $\beta \in \mathbb{Z}_{-N_\beta}^{N_\beta}$, when $N_\beta \geq 3$, leaving $2N_\beta + 1$ summands\footnote{In fact, the accuracy of this approximation also depends on $\sigma_{\theta_o}$ and the error $e_\beta$ (i.e., the absolute difference from the true value) grows when $\sigma_{\theta_o}$ increases. For example, let $N_\beta = 3$ and $\sigma_{\theta_o} = 1$, then $e_\beta < 10^{-15}$, if $\sigma_{\theta_o}$ = 5, then $e_\beta < 10^{-6}$ \cite{kurz_wrapped_gauss}. Generally, for realistic values of $\sigma_{\theta_o}$, i.e., $\sigma_{\theta_o} \leq \pi$, the error is insignificant, especially when comparing it to the precision of numerical integration, which we will later set to three digits.}. While (\ref{eq_polar}) requires numerical integration (see \cite{tolksdorf2024}), we can analytically integrate (\ref{eq_inf_wrapped_gauss}) over $\tilde{\mathcal{A}}_{\theta_o}(\phi, \rho)$ in (\ref{eq:collision_set}), given the approximation $\beta \in \mathbb{Z}_{-N_\beta}^{N_\beta}$. Hence, with (\ref{eq:poc}), we obtain the POC estimate as 
\begin{equation}\label{eq_POC_Gaussian}
\begin{split}
    &\mathbb{P}\{\mathbf{y}_{o,p} \in \tilde{\mathcal{A}}_{cir}\} = \frac{1}{2}\int_0^{2\pi} \int_{0}^{\overline{\rho}}p_{\phi_o, \rho_o}(\phi, \rho) \sum_{i = 1}^{N_I} \sum_{\beta=-N_\beta}^{N_\beta}\Bigg[\\
    &\text{erf}\left( \frac{\overline{\theta}_i(\phi, \rho)\!-\!\mu_{\theta_o}\!+\!2\pi \beta}{\sigma_{\theta_o} \sqrt{2}}\right)\!- \! \text{erf}\left( \frac{\underline{\theta}_i(\phi, \rho)\!-\!\mu_{\theta_o}\!+\!2\pi \beta}{\sigma_{\theta_o} \sqrt{2}}\right)\Bigg]\\
    &\text{d}\rho\text{d}\phi,
    \end{split}
\end{equation}
 where $\text{erf}(z)$ denotes the error function of $z$. Also recall that $\overline{\theta}_i(\phi, \rho)$ and $\underline{\theta}_i(\phi, \rho)$ are the boundaries of the intervals of $\tilde{\mathbb{I}}^{\theta_o}_{i}(\phi, \rho)$ in (\ref{eq:poc}). \\
 \subsection{POC Estimation Algorithm}
 The computational efficiency of the algorithm depends on 1) the computation of the intersection angle intervals $\mathbb{I}^{\theta_o}_{l,q}(\phi, \rho)$ and 2) sorting those into pairwise disjoint intervals $\tilde{\mathbb{I}}^{\theta_o}_{i}(\phi, \rho)$. Here, Remark \ref{rem:intervals}, outlines an approach to an efficient implementation. Crucially, due to the analytical solution of the PDF in the object's heading angle, the computation of the POC only requires one numerical integration of the remaining two-dimensional integral with respect to the variables $\rho$ and $\phi$ in (\ref{eq_POC_Gaussian}). Here, we find that if the sampling ranges (i.e., discretization intervals of a Riemannian approximation of the integral) of $\phi, \rho$ for numerical integration are kept constant, which only need to change if the numerical integration tolerance is altered, all intersection angle intervals $\tilde{\mathbb{I}}^{\theta_o}_{i}(\phi, \rho)$ can be pre-computed for the desired numerical integration accuracy. Namely, those intervals do \textit{not} dependent on the object's configuration. \\
 Algorithm \ref{alg:poc_overview} overviews our proposed POC estimate computation algorithm for numerical integration with the trapezoidal rule. Our proposed algorithm leverages vectorization, i.e., performing many element-wise operations in one line of programming code. If a variable in Algorithm \ref{alg:poc_overview} - \ref{alg:SortInts} is vectorized, and, hence, all operations in that line are vectorized, then the variable is identified by an arrow, e.g., $\vec{\phi}$. Element-wise logic for vectorized variables is denoted by a logic operation in square brackets, e.g., $\vec{\underline{\theta}}[\vec{\rho}' \leq R - L_o] \gets 0 $ denotes that for all elements in $\vec{\underline{\theta}}$, where, for the corresponding element in $\vec{\rho}'$ (note that $\vec{\rho}'$  and $\vec{\underline{\theta}}$ are of the same length), $\vec{\rho}' \leq R - L_o$ is true, $0$ is assigned. To perform the two-dimensional integral over $\phi, \rho$ in (\ref{eq_POC_Gaussian}) with the trapezoidal rule, a grid over $[0, 2\pi] \times [0, \overline{\rho}]$ must be sampled. Therefore, in line 2 and 3 of Algorithm \ref{alg:poc_overview}, we sample $[0, \Delta \phi, 2\Delta\phi, ..., (N_s-1)\Delta\phi, 2\pi] \times [0, \Delta \rho, 2\Delta\rho, ..., (N_s-1)\Delta\rho, \overline{\rho}]$  so that $\vec{\phi}$ and $\vec{\rho}$ are each of length $N_s \times N_s$  and row-by-row of $(\vec{\phi}, \vec{\rho})$ unique pairs are given. Line 4 of Algorithm \ref{alg:poc_overview} calls another function, given in Algorithm \ref{alg:CollAngleInts}, which calculates the intersection angle intervals $\mathbb{I}^{\theta_o}_{l,q}(\phi, \rho)$ and stores them in a matrix $M_I$. Here, for each circle-to-circle pair, the intersection angle interval is computed. The inner loop in lines 13-37 computes the intervals for pairs of object circles, as outlined by Remark \ref{rem:intervals}. Note that the function $ \text{GetDistance}(N_{c,e}, d_{c,e}, l)$ computes the offsets $L_e$ and $L_o$ for the circle of consideration. After obtaining the matrix of all intersection angle intervals $M_I$, Algorithm \ref{alg:SortInts} ensures that for each sampled pair of $\phi, \rho$, only disjoint intersection angle intervals remain, see (\ref{eq:collision_set}). Whenever two intervals overlap, the union of both is taken. After taking the union, it is also ensured that the maximum length of the remaining interval does not exceed $2\pi$ (see line 8 in Algorithm \ref{alg:SortInts}). Note that lines 6 - 8 in Algorithm~\ref{alg:SortInts} must be performed in modulo $2\pi$ arithmetic, which requires checking whether one or both intersection angle intervals wrap around.\\
The main advantage of the proposed POC algorithm is that lines 1-5 in Algorithm \ref{alg:poc_overview}, which carry most of the computational burden, can be calculated at initialization, as the object configuration and it's uncertainty is not needed. Lines 6-34 estimate the two-dimensional integral (\ref{eq_POC_Gaussian}), for a given object mean $\mu_o = (\mu_{x_o}, \mu_{y_o}, \mu_{\theta_o})$ and variance $\Sigma_o = (\sigma_{x_o}^2,\sigma_{y_o}^2, \sigma_{\theta_o}^2)$. Note that the summations in  (\ref{eq_POC_Gaussian}), i.e., the analytic solution of the PDF in the object's heading angle, are calculated in lines 10-17 with vectorization, while the remaining numerical integration (lines 19-32) are non-vectorized. Therefore, the subscript $m +N_s(j-1)$ in line 23 denotes the $m +N_s(j-1)$-th element in $\vec{P}_\theta$.

\begin{algorithm}
\caption{Multi-circle Probability of Collision}\label{alg:poc_overview}
\begin{algorithmic}[1]
\Require $N_{c,e}, N_{c,o}, l_e, w_e, l_o, w_o$ \Comment{geometric inputs}
\State $r_e, r_o, d_{c,e}, d_{c,o} \gets N_{c,e}, N_{c,o}, l_e, w_e, l_o, w_o$  \Comment{see (\ref{eq_placing})}
\State $\vec{\phi} \gets N_{s}$ \Comment{sample $N_{s} \times N_{s}$ values of $\phi$}
\State $\vec{\rho} \gets N_{s}, \overline{\rho}$ \Comment{sample $N_{s} \times N_s$ values of $\rho$}
\State $M_I \gets \text{InterAngleInts}(N_{c,e}, N_{c,o}, r_e, r_o, d_{c,e}, d_{c,o},\vec{\phi}, \vec{\rho})$
\State $\tilde{M}_I \gets \text{SortInts}(M_I, N_{io}, N_{c,e})$
\Require $\mathbf{\mu_o}, \Sigma_o$ \Comment{input object mean and variance}
\State $\mu_{\rho_o} \gets \sqrt{\mu_{x_o}^2 + \mu_{y_o}^2}$
\State $\mu_{\phi_o} \gets \text{mod}(\text{atan2}(\mu_{y_o}, \mu_{x_o}), 2\pi)$
\State $\mu_{\theta_o} \gets \text{mod}(\mu_{\theta_o}, 2\pi)$
\State $\vec{P}_{\theta_o} \gets \vec{0}$ 
	\For{$i \gets 0$ to $N_I$}\Comment{do summations in (\ref{eq_POC_Gaussian})}
	\For{$\beta \gets -N_\beta$ to $N_\beta$}
		\State $\text{extract } \vec{\tilde{\mathbb{I}}}_{i} \text{ from } \tilde{M}_I$
		\State $\vec{P}_1 \gets \text{erf}\left( \frac{\vec{\overline{\theta}}_i-\mu_{\theta_o}+2\pi \beta}{\sigma_{\theta_o} \sqrt{2}}\right)$
		\State $\vec{P}_2 \gets \text{erf}\left( \frac{\vec{\underline{\theta}}_i-\mu_{\theta_o}+2\pi \beta}{\sigma_{\theta_o} \sqrt{2}}\right)$
		\State $\vec{P}_{\theta_o} \gets \vec{P}_1 - \vec{P}_2 + \vec{P}_{\theta_o}$
	\EndFor
	\EndFor
\State $P_{\phi_o, \rho_o, \theta_o} \gets 0$
\For{$j \gets 1$ to $ N_s$}
\State $\phi_j \in [0, \Delta \phi, 2\Delta\phi, ..., (N_s-1)\Delta\phi, 2\pi]$ 
\For{$m \gets 1$ to $N_s$}
\State $\rho_m \in [0, \Delta \rho, 2\Delta\rho, ..., (N_s-1)\Delta\rho, \overline{\rho}]$
	
	\State $P_3 \gets  \frac{1}{2}p_{\phi_o, \rho_o}(\phi_j, \rho_m) P_{\theta, m + N_s(j-1)}$ \Comment{see (\ref{eq_POC_Gaussian})}
	\If{$j == 1$ or $j == N_s$}
		\State $P_3 \gets \frac{1}{2}P_3$ \Comment{corners of integration grid}
	\EndIf
	\If{$m == 1$ or $m == N_s$}
		\State $P_3 \gets \frac{1}{2}P_3$ \Comment{corners of integration grid}
	\EndIf
	\State $P_{\phi_o, \rho_o} \gets P_3 +  P_{\phi_o, \rho_o, \theta_o}$
\EndFor
\EndFor
\State $\mathbb{P}\{\mathbf{y}_{o,p} \in \tilde{\mathcal{A}}_{cir}\} \gets \Delta \phi \Delta \rho P_{\phi_o, \rho_o, \theta_o}$ \Comment{weigh by sizes}\\
\Return $\mathbb{P}\{\mathbf{y}_{o,p} \in \tilde{\mathcal{A}}_{cir}\}$
\end{algorithmic}
\end{algorithm}

\begin{algorithm}
\caption{Intersection Angle Intervals `InterAngleInts'}\label{alg:CollAngleInts}
\begin{algorithmic}[1]
\Require $N_{c,e}, N_{c,o}, r_e, r_o, d_{c,e}, d_{c,o}, \vec{\phi}, \vec{\rho}$
\If{$N_{c,o}$ is even}
	\State $N_{io} \gets N_{c,o}/2$
\Else
	\State $N_{io} \gets (N_{c,o}+1)/2$
\EndIf
\State $\vec{x} \gets \vec{\rho} \cos{(\vec{\phi})}$
\State $\vec{y} \gets \vec{\rho} \sin{(\vec{\phi})}$
\State $\text{initialize matrix } M_I \text{ of size } N_s \times 4N_{io}N_{c,e}$
\For{$l \gets 1$ to $N_{c,e}$}
	\State $L_e\gets \text{GetDistance}(N_{c,e}, d_{c,e}, l)$
	\State $\vec{\rho}' \gets \sqrt{(\vec{x} -  L_e)^2 + \vec{y}^2}$ \Comment{see (\ref{eq:polar_shift})}
	\State $\vec{\phi}' \gets \text{mod}(\text{atan2}(\vec{y}, \vec{x}) , 2\pi)$ \Comment{see (\ref{eq:polar_shift})}
	\For{$q \gets 1$ to $N_{io}$}
		\State $ L_o \gets \text{GetDistance}( N_{c,o}, d_{c,o}, q)$
		\If{$L_o == 0$} \Comment{check if circle is not offset}
			\State $[\vec{\underline{\theta}}, \vec{\overline{\theta}}] \gets [\vec{0}, \vec{2\pi}]$
			\State $\vec{\overline{\theta}}[\vec{\rho}' > R] \gets 0$
			\State $M_I \gets \vec{\mathbb{I}}^{\theta_o}_{l,q} := [\vec{\underline{\theta}}, \vec{\overline{\theta}}]$
		\Else
			\State $\vec{\theta}_{coll} \gets \arccos{\left(\frac{L_o^2 + \vec{\rho}'^2 - R^2}{2L_o\vec{\rho}'}\right)}$ \Comment{see (\ref{eq_circle_v_offset_circle_heading_bounds})}
			\State $\vec{\theta}_{coll} [\vec{\rho}' \geq R + L_o] \gets 0 $ \Comment{see Table \ref{tab_collision_conditions_circle_v_circle}}
			\State $\vec{\overline{\theta}} \gets \vec{\phi}' + \pi + \vec{\theta}_{coll}$ \Comment{see (\ref{eq_circle_v_offset_circle_heading_bounds})}
			\State $\vec{\underline{\theta}} \gets \vec{\phi}' + \pi - \vec{\theta}_{coll}$ \Comment{see (\ref{eq_circle_v_offset_circle_heading_bounds})}

			\If{$R \geq L_o$} \Comment{see Table \ref{tab_collision_conditions_circle_v_circle}}
				\State $\vec{\overline{\theta}}[\vec{\rho}' \leq R - L_o] \gets 2\pi $
				\State $\vec{\underline{\theta}}[\vec{\rho}' \leq R - L_o] \gets 0 $
			\Else \Comment{see Table \ref{tab_collision_conditions_circle_v_circle}}
				\State $\vec{\overline{\theta}}[\vec{\rho}' \leq L_o - R] \gets 0 $
				\State $\vec{\underline{\theta}}[\vec{\rho}' \leq L_o - R] \gets 0 $		
			\EndIf	

			\State $\vec{\overline{\theta}}[\vec{\overline{\theta}} - \underline{\vec{\theta}} > \pi] \gets \pi $
			\State $\vec{\underline{\theta}}[\vec{\overline{\theta}} - \underline{\vec{\theta}} > \pi] \gets 0 $
			\State $\vec{\mathbb{I}}^{\theta_o}_{l,q} \gets [\vec{\underline{\theta}}, \vec{\overline{\theta}}]$
			\State $\vec{\mathbb{I}}^{\theta_o}_{l,N_{c,o} - q} \gets [\vec{\underline{\theta}} + \pi, \vec{\overline{\theta}} + \pi]$ \Comment{see Remark \ref{rem:intervals}}
			\State $M_I \gets \vec{\mathbb{I}}^{\theta_o}_{l,q}, \vec{\mathbb{I}}^{\theta_o}_{l,N_{c,o} - q}$
			\EndIf	
	\EndFor
\EndFor\\
\Return $M_I, N_{io}$
\end{algorithmic}
\end{algorithm}

\begin{algorithm}
\caption{Sort Intervals `SortInts'}\label{alg:SortInts}
\begin{algorithmic}[1]
\Require $M_I, N_{io}, N_{c,e}$
\State $N_I \gets 2N_{io}N_{c,e}$ \Comment{number of interval vectors}
\For{$i \gets 1$ to $N_I$}
	\For{$j \gets 1$ to $N_I$}
\If{$i \neq j$}
	\State $\text{extract intervals } \vec{\mathbb{I}}^{\theta_o}_{i} \text{ and } \vec{\mathbb{I}}^{\theta_o}_{j} \text{ from } M_I$
	\State $\vec{\mathbb{I}}^{\theta_o}_{i}[\vec{\mathbb{I}}^{\theta_o}_{i} \cap \vec{\mathbb{I}}^{\theta_o}_{j} \neq \emptyset] \gets \mathbb{I}^{\theta_o}_{i} \cup \mathbb{I}_{j}$\Comment{check overlap}
	\State $\vec{\mathbb{I}}^{\theta_o}_{j}[\vec{\mathbb{I}}^{\theta_o}_{i} \cap \vec{\mathbb{I}}^{\theta_o}_{j} \neq \emptyset] \gets [0,0]$
	\State $\vec{\mathbb{I}}^{\theta_o}_{i}[\vec{\overline{\theta}}_i - \vec{\underline{\theta}}_i > 2\pi] \gets [0, 2\pi]$ \Comment{check length}
	\State $\text{sort }\vec{\mathbb{I}}^{\theta_o}_{i} \text{ in the } i\text{-th column of } M_I$
	\State $\text{sort }\vec{\mathbb{I}}^{\theta_o}_{j} \text{ in the } j\text{-th column of } M_I$
\EndIf
\EndFor
\EndFor\\
\Return $\tilde{M}_I \gets M_I$
\end{algorithmic}
\end{algorithm}

\section{Motion Planning by Stochastic Model Predictive Control}\label{sec_SMPC}
 One of the motivations for the analytical derivation of an over-approximate POC (\ref{eq_POC_Gaussian}), is the (real-time) application in optimization-based motion planning algorithms. Consider the path-following SMPC problem described hereafter.
 \subsection{Path-Following Problem}
 A reference path $\mathcal{P}$, i.e., a regular curve in the configuration space $\mathcal{C}$, and a constant reference velocity $v_{ref} \in \mathbb{R}_{\geq0}$, are provided to the ego. The ego plans its own motion online, i.e., during run-time, by minimizing the error with respect to both references within some finite horizon $\mathbb{Z}_{k}^{k +N_P}$, with $N_P \in \mathbb{N}_{>0}$. We assume that the references can be realized exactly by the vehicle, i.e., the path and velocity references satisfy the physical limitations of the ego vehicle dynamics. \\
Path-following problems treat the time evolution along the reference path as an additional degree of freedom,enabling the controller to plan maneuvers that may deviate from time-independent references to satisfy constraints, see, e.g., \cite{Faulwasser.2016, tolksdorf2023risk}. Essentially, the planner can freely maneuver by accepting costs from reference errors to satisfy constraints, such as the POC constraint in our setting. Let the motion dynamics of the ego  be modeled by a discrete-time non-linear system of the form (e.g., obtained by time discretization of a continuous-time vehicle dynamics model):
\begin{equation}\label{eq:sys}
\begin{split}
\mathbf{z}_{k+1} &= f(\mathbf{z}_k, \mathbf{u}_{k}), \\
(\mathbf{y}^T_{e, k}, v_{e,k}) & = h(\mathbf{z}_{k}),
\end{split}
\end{equation}
where $\mathbf{z}_k \in \mathbb{R}^{n_z}$ represents the ego's state vector and $\mathbf{u}_{k} \in \mathbb{R}^{n_{u}}$ is the input vector at time $k$. Further, $\mathbf{y}^T_{e, k}, v_{e,k}$ are the ego's configuration and longitudinal velocity at time $k$, respectively. The model equations are given by $f$, and $h$ is the output function. State and input constraints are given by bounding $\mathbf{z}_k$ and $\mathbf{u}_{k}$ in sets, i.e., $\mathbf{z}_k \in \mathcal{Z}$ and $\mathbf{u}_{k} \in \mathcal{U}$, respectively, which must be satisfied at all times. The path is a regular curve parameterized by a scalar variable $\lambda \in [\lambda_0, \lambda_g]$ (see \cite{mcnaughton2011motion} for ways to compute a path from a function giving the curvature.). At each planning cycle, the initial reference point on the path is determined by finding the closest point to the ego, i.e.,
\begin{equation}
\lambda_{k|k} = \argmin_{\lambda \in [\lambda_0, \lambda_g]} \| \mathbf{y}_{e, k|k} - \mathbf{y}_P(\lambda) \|, 
\end{equation}
where $\mathbf{y}_P(\lambda) \in \mathcal{P}$. Within the planning horizon, we use the ego's velocity heading to approximate the progress along the path with 
\begin{equation}\label{eq:input}
\lambda_{n+1|k} = \lambda_{n|k} + v_{e,n|k}\cos{(\theta_{e,k|n} - \theta_{p,n|k})},
\end{equation}
where $\theta_p$ is the reference heading angle, i.e., the angle of the tangent to the path at $\lambda_{k|n}$ (see Figure 2 in~\cite{tolksdorf2023risk}). Note that the accuracy of the approximation (\ref{eq:input}) diminishes as the path's curvature increases. As such, we define the path-following error as 
\begin{equation}\label{eq:error}
\begin{split}
\mathbf{e}(\lambda_{n|k}) & =
\begin{pmatrix}
\mathbf{y}^T_{e, n|k} - \mathbf{y}_P^T(\lambda_{n|k})\\
v_{e,n|k} - v_{ref}
\end{pmatrix}.
\end{split}
\end{equation}
Accordingly, when $\mathbf{e}_k=\mathbf{0}$, the ego follows the reference path with the reference velocity, because the controller's chosen velocity $v_{e,n|k}$ is mapped onto the progress of the path with (\ref{eq:input}). As a cost function $J$, we quadratically penalized deviations of the reference path and velocity multiplied by a positive definite weighting matrix $W \in \mathbb{R}^ {4\times4}$ as 
\begin{equation}
    J(\lambda_{n|k}) =\mathbf{e}^T(\lambda_{n|k}) W\mathbf{e}(\lambda_{n|k}).
\end{equation}
With the cost function $J$ and input sequence $\mathbf{U}_k = [\mathbf{u}_{k|k}, \mathbf{u}_{k|k+1}, \dots, \mathbf{u}_{k|k+N_P}]$ over the planning horizon $\mathbb{Z}_{k}^{k +N_P}$, we define the path-following SMPC problem as
\begin{subequations}
\begin{equation}
     \begin{split}
&V^{SMPC}(\mathbf{z}_{k}, \lambda_k) := \min_{\mathbf{U}_k} \sum_{n=k}^{k + N_p} J(\lambda_{n|k})\label{eq:ra}, \\
\end{split}
\end{equation}
 subject to: 
\begin{align}
 & \;\;\;\; \mathbf{z}_{k|k} = \mathbf{z}_k, \;\;\;\; \lambda_{k|k} = \lambda_k,\label{rb} \\
& \forall n \in \mathbb{Z}_k^{k+N_P}: \mathbf{z}_{n+1|k} = f(\mathbf{z}_{n|k}, \mathbf{u}_{n|k}),  \label{rc}\\ 
& \;\;\;\; (\mathbf{y}_{e,n|k}^T, v_{e, n|k})  = h(\mathbf{z}_{n|k}),\label{rd}\\
& \;\;\;\; \lambda_{n+1|k} =  \lambda_{n|k} + v_{e,n|k}\cos{(\theta_{e,k|n} - \theta_{p,n|k})}, \label{re}\\ 
& \;\;\;\; \mathbf{e}(\lambda_{n|k}) = (\mathbf{y}_{e,n|k}^T - \mathbf{y}_P^T(\lambda_{n|k}), v_{e,n|k} - v_{ref})^T,\label{rf} \\
&\;\;\;\;  \mathbf{z}_{n|k} \in \mathcal{Z}, \;\;\;\;\mathbf{u}_{n|k} \in \mathcal{U}, \label{rg}\\ 
&\quad \mathbb{P}\{\mathbf{y}_{o,p, n|k} \in \tilde{\mathcal{A}}_{cir}\} \leq \epsilon.\label{rh}
\end{align}
\end{subequations}
Here, the initial conditions are set in (\ref{rb}). The system model is provided in (\ref{rc}) - (\ref{re}), from which the error is calculated in (\ref{rf}) to support path-following by minimization of this error. State and input constraints are given in (\ref{rg}). Lastly, the POC of (\ref{eq_POC_Gaussian}) is calculated each time-step and constrained by a constant POC tolerance $\epsilon$ in (\ref{rh}) to warrant a desired level of safety. Note that implementing the POC as a constraint is common practice in the literature, see, e.g., \cite{schwarting2017safe, goulet2022probabilistic, schreier2016integrated, mustafa2024racp}.

\section{Numerical Case Studies}\label{sec_simulation_example}
This section demonstrates and analyses the proposed approach in two parts. First, the POC estimation approach proposed in Sections \ref{sec_collision_events} - \ref{sec_gaussian_POC_algo} is analyzed and compared to MCS in terms of estimation accuracy and computational efficiency. Second, the path-following SMPC with the POC constraint, as described in Section \ref{sec_SMPC} and using the proposed POC approximation, is demonstrated, where we also compare to the case using MCS.

\subsection{POC Estimation}\label{sec_poc_estimation}
We choose three representative scenarios to demonstrate the proposed method from Section \ref{sec_gaussian_POC_algo}, where the following simulations \textit{do not} use the controller proposed in Section \ref{sec_SMPC}. The first two scenarios, \textit{intersection collision} and \textit{intersection pass}, respectively, represent intersection scenarios, whereas the third scenario, referred to as \textit{oncoming pass}, represents a passing scenario of an oncoming vehicle on a straight road. For all scenarios, we let all vehicles travel at a constant velocity. Regarding the intersection scenarios, the trajectories are perpendicular, meaning their paths intersect. In the \textit{intersection collision} scenario, the actors collide; since the collision itself is not modeled, both vehicles will drive through each other in case of a collision. Here, it is expected that the POC must approach one during the vehicles' overlap (collision) if the uncertainty associated with the objects configuration is sufficiently small\footnote{Suppose the uncertainties $\sigma_{x_o}, \sigma_{y_o}, \sigma_{\theta_o}$ are larger than the respective bounds of $\tilde{\mathcal{A}}_{cir}$, than the POC is less than one, even if the mean values would indicate a collision.}. In the \textit{intersection pass} scenario, the object will pass the intersection before the ego vehicle. In the \textit{oncoming pass} scenario, the object and the ego vehicle pass each other in the oncoming lanes; here, the width of the vehicles must be tightly accounted for by the shape approximations. Note that in each simulation, we approximate both vehicles with the same number of circles; e.g., if two circles approximate both vehicles, we call this the two-circle case. We choose the length and width of a standard European mid-sized car to parameterize both rectangles. To parameterize (\ref{eq_POC_Gaussian}), we select the object's ground-truth configurations as mean values and model the standard deviations $\Sigma_o = (\sigma_{x_o}, \sigma_{y_o}, \sigma_{\theta_o})$ by a distance-dependent logistic function. This simulates the effect that the ego generally estimates configurations more accurately of closer objects (note that the case of fixed configurations and various constant uncertainties is analyzed in Section \ref{sec:poc_accuracy}). Hence, the standard deviations are given by: 
\begin{equation*}
    \Sigma_o(d_k) = \frac{1}{1+ \exp{[-\gamma (d_k - d_0)]}}\Sigma_{max},
\end{equation*}
where $\gamma, d_0$ are free parameters, $d_k=\|\mathbf{x}_{k,e}-\mathbf{x}_{k,o}\|$ and $\Sigma_{max} \in \mathbb{R}^{3}$ has positive entries representing the maximum standard deviations. Additionally to the POC estimation for a varying number of circles, we also compute the ground truth POC by applying MCS on both rectangles. For the MCS sampling on rectangles, the separating axis theorem \cite{Boyd_Vandenberghe_2004} is used to check for collisions. All simulation parameters can be found in the Appendix \ref{appendixB}. 
\subsubsection{POC Simulation Results}
The time evolution of the POCs for four different circular approximations are presented in Figure \ref{results_all_scenarios}.
\begin{figure}[t!]
\begin{center}
\includegraphics[width=9cm]{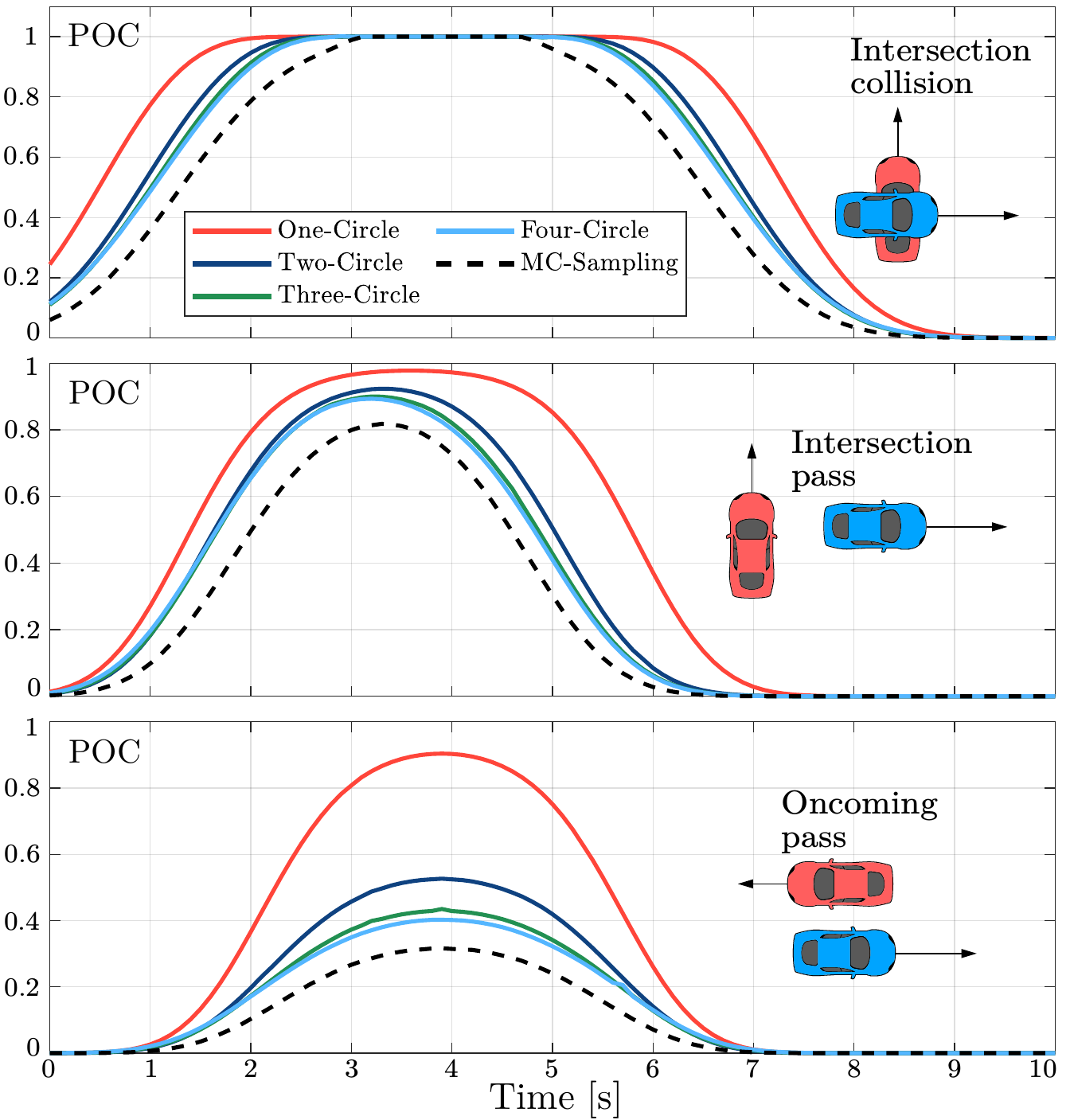}  
\caption{Resulting POCs for varying amount of circles in three scenarios; the ego vehicle is depicted in blue and the object vehicle in red.}
\label{results_all_scenarios}
\end{center}
\end{figure}
The number of samples for the MCS method is set to $10^6$, so the MC-sampling results can be considered the ground-truth-POC\footnote{For MC-Sampling, the error scales by a factor of $1/\sqrt{N}$, where  $N$ is the number of samples. Hence, for $10^6$ samples the error is $< 10^{-3}$.} for rectangular shapes. In the \textit{intersection collision} and \textit{intersection pass} scenarios, the differences between the two-circle and four-circle cases are marginal. However, in the \textit{oncoming pass} scenario, the difference between the two-circle and three-circle case is noticeable, peaking at approximately ten percentage points. Generally, the difference between the three-circle and the four-circle case is insignificant for the given scenarios and vehicle shapes. Figure \ref{results_all_scenarios} also shows that the proposed POC estimate over-approximates the real POC and that the approximation can represent the evolution of the actual POC well in the chosen scenarios. Summarizing, the approximation error depends on the scenario parameters, i.e., the constellation of how both vehicles approach each other. Further, beyond the three-circle approximation, the advantages of adding more circles vanish for the given vehicle sizes. 
\begin{figure}[t!]
\begin{center}
\includegraphics[width=8.4cm]{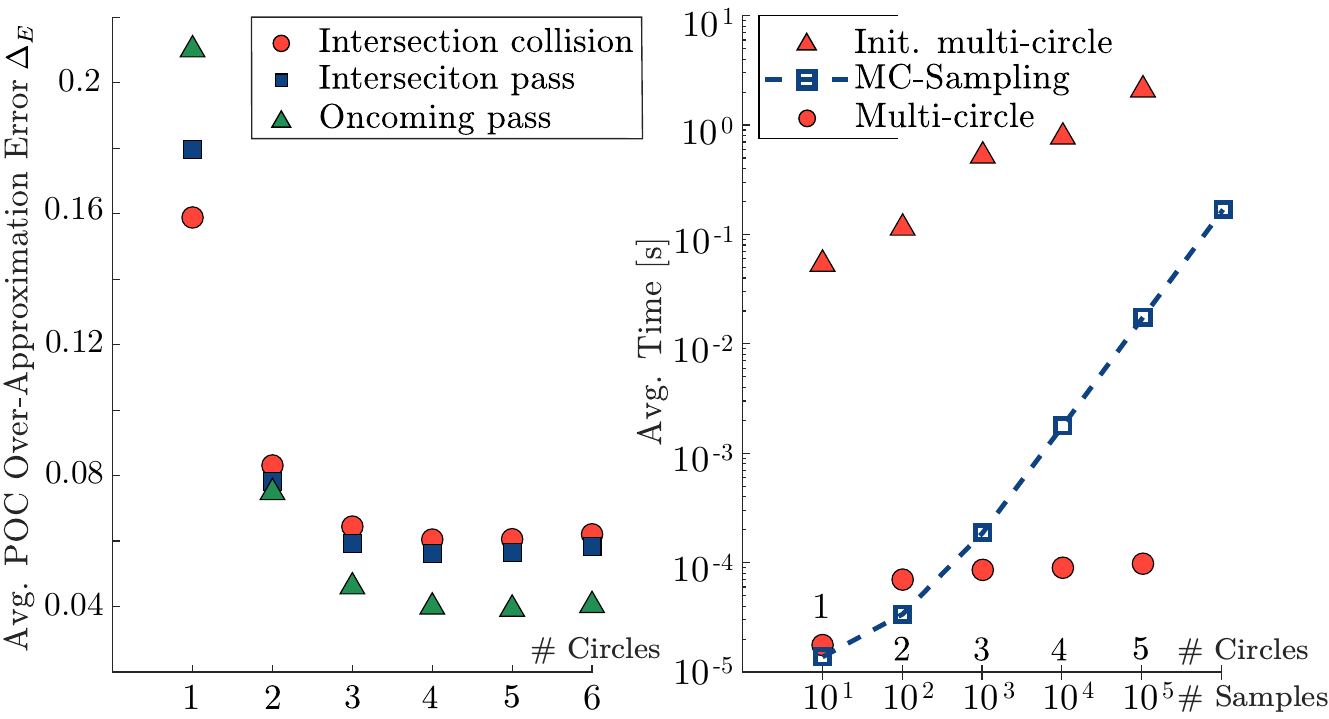}  
\caption{POC over-approximation error (left) and average runtime (right). In the right figure, the red circles and triangles relate to the number of circles on the $x$-axis, whereas MCS relates to the number of samples on the $x$-axis.}
\label{fig_error_runtime}
\end{center}
\end{figure}
\begin{figure}
\begin{center}
\includegraphics[width=8.4cm]{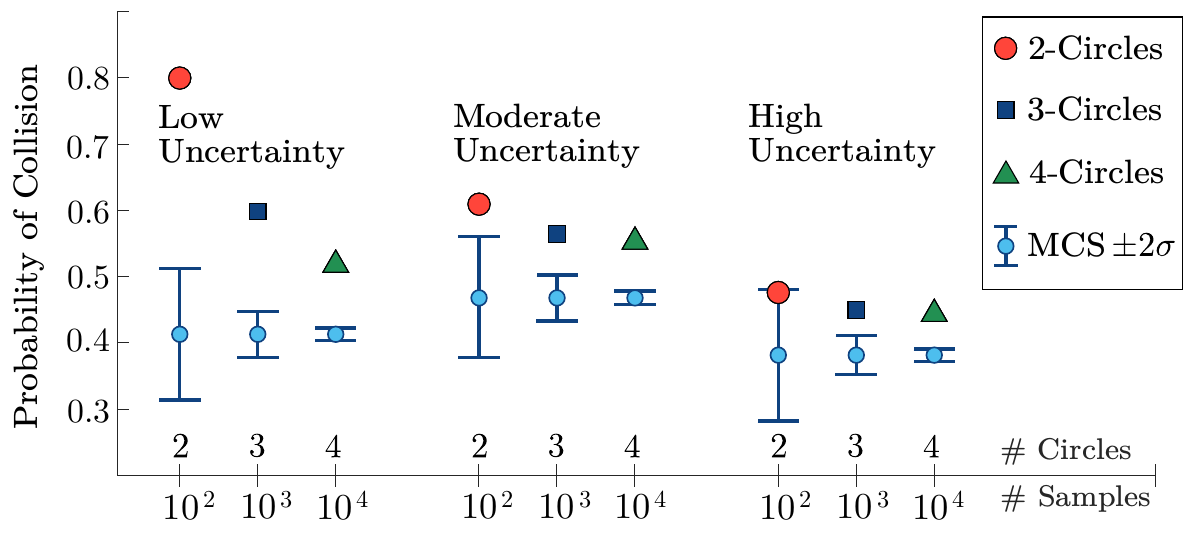}  
\caption{Estimation accuracy of the multi-circle approach and MCS for varying levels of uncertainty. Note that the MCS fluctuations, highlighted by the $\pm2\sigma$ error bars, for $10^3$ samples are too large for optimization-based motion planning (see Figure~\ref{fig_SMPC_MCS}).}
\label{fig_accuracy}
\end{center}
\end{figure}
\subsubsection{POC Estimation Accuracy and Computational Efficiency}\label{sec:poc_accuracy}
Clearly, the computational scaling of the POC algorithm stands in a trade-off with estimation accuracy. The scaling depends on the amount of intersection angle intervals $N_I$ and samples for numerical integration $N_s$. We achieved the best performance when Algorithm \ref{alg:poc_overview} is implemented in C$++$ in two steps\footnote{The programming code can be accessed online here: \url{https://github.com/Tolksdorf/Collision-Probaility-Estimation.git}.}: first, lines 1-5 are computed using the Eigen library \cite{eigenweb} for vectorization. Second, lines 6-34 are programmed symbolically with CasADi \cite{Andersson2019}. This choice yields excellent performance, as a symbolic expression for $\mathbb{P}\{\mathbf{y}_{o,p} \in \tilde{\mathcal{A}}_{cir}\}$ is pre-computed at initialization. To retrieve a POC value at run-time, the object mean $\mu_o$ and variance $\Sigma_o$ are substituted into the symbolic expression, and the expression, which is essentially an extensive summation, is evaluated. We implemented MCS on two colliding rectangles to assess the computational efficiency as a benchmark with the algorithm proposed by \cite{lambert2008collision}. We assessed several implementations and found that Matlab's automatic C$++$ code generation was the most efficient. We used a notebook with an Intel i7-9850H processor and \qty{32}{GB} of memory as a computational platform. The average time over $10^4$ POC computations for a given multi-circle case is provided in Figure \ref{fig_error_runtime} (right figure), for $N_s \times N_s = 20 \times 20$ samples. Here, we separated the initialization time (red triangles) from the POC estimation time for a given mean $\mu_o$ and variance $\Sigma_o$ (red dots) because the initialization must only be performed once for a given geometry and desired accuracy. For each POC estimation, we randomly drew a mean $\mu_o$ and variance $\Sigma_o$ from a uniform distribution so that each POC estimation featured distinctly different values for  $\mu_o$ and $\Sigma_o$. The absolute computation time for the initialization and POC estimation mainly depends on the number of samples $N_s$ the two-dimensional integral is evaluated over; the number of circles changes the POC estimation time insignificantly.  As expected, the scaling of estimating the POC with MCS behaves linearly according to the number of Monte Carlo samples. For more than $10^3$ Monte Carlo samples, which are typically needed for accurate POC estimation and motion planning (as we will show later), the MCS approach takes significantly longer to compute than the multi-circle approach. For example, $10^4$ Monte Carlo samples compute more than 20 times slower than any number of circles in Figure \ref{fig_error_runtime} (right).\\
The left figure in Figure \ref{fig_error_runtime} shows the average over-approximation error given the all the one-circle to six-circle cases. For that analysis, we considered MCS on two rectangles with $10^6$ samples as the ground truth POC. With this information at hand, we can compute the average error $\Delta_E$ for $T$ simulation time-steps as
\begin{equation*}
    \Delta_E = \frac{1}{T}\sum_{k= 0}^T (\mathbb{P}_k\{\mathbf{y}_{o,p} \in \tilde{\mathcal{A}}_{cir}\} - \mathbb{P}_k\{\mathbf{y}_o \in \tilde{\mathcal{A}}_{rect}\}).
\end{equation*}
Figure \ref{fig_error_runtime} (left) shows that placing more than three circles on each actor yields insignificant average estimation accuracy differences. Also, increasing the number of circles does not necessarily imply a lower approximation error, as geometric conditions (\ref{eq_placing}) imply the smallest possible radii given a number of circles and not the smallest enclosure for any number of circles. An inherent disadvantage of MCS is that the result fluctuates around the ground truth POC, which is generally undesired from a safety and optimization algorithm standpoint. Contrarily, the proposed POC algorithm is deterministic concerning the input parameters, i.e., for given uncertainty parameterization and vehicle configurations, the algorithm will always return identical values and it is guaranteed to be an over-approximate. This is illustrated in Figure \ref{fig_accuracy}. Here, we fixed the vehicle's configurations and set three different uncertainties, low, moderate, and high, respectively (see Appendix \ref{appendixC} for the associated numerical values). We used three different numbers of MCS samples for each uncertainty level and three different circle-to-circle approximations. For MCS, we repeated the estimation $10^4$ times and calculated the standard deviation $\sigma$ to assess the fluctuations caused by the MCS method. As expected, more circles give a more accurate POC estimation in case of low to moderate uncertainty, though the advantage diminishes for more than three circles for higher uncertainties. The MCS fluctuations are depicted by the $\pm 2\sigma$ error bars. The fluctuations are significant for less than $10^4$ MCS samples, confirming suspicions that the result can be severely under-approximated for fewer samples. In conclusion, the MCS fluctuations are manageable for more than $10^4$ samples, which is about a \textit{factor of 23 slower} to compute than the multi-circle approach with a three-circle approximation for our implementations. The three-circle approximation yields a satisfactory over-approximation error, small enough to, e.g., maneuver around other vehicles at close distances. 
 
\subsection{Stochastic Model Predictive Control}\label{sec_results_SMPC}
The main challenge of SMPC, regardless of the POC estimation method, is ensuring recursive feasibility. While there are no theoretical guarantees of recursive feasibility (i.e., for any PDF), practitioners are also dealing with challenges in the numerical implementation, e.g., solvers getting stuck in local minima. Here, solvers which can deal with non-convex problems are often unpractical with respect to computational efficiency. We used the software Matlab, and the recommendations of \cite{matlab_optimization_documentation} to \textit{robustify} the SMPC controller against the problem of getting stuck in local minima. For the MCS approach, we expect, however, that the fluctuation of the MCS result for finite samples (see Figure \ref{fig_accuracy}) will cause further feasibility issues. 
 
\subsubsection{Simulation Scenario}
We chose an overtaking scenario to demonstrate the benefits of our proposed POC estimation approach in the SMPC context. The ego's path is a straight line parallel to the global $x$-axis. The ego is tasked with following it with a reference velocity $v_{ref} =$ \qty{6}{m/s}. Some space ahead, an object is also following the path with a velocity of $v_o =$ \qty{2}{m/s}. Therefore, the path-following SMPC must balance the path-following error with the velocity error while adhering to the POC constraint. By weighting the velocity error higher than the positional error, it is expected that the SMPC should generate an overtaking maneuver. The ego detects the object's configuration with some initial uncertainty $\Sigma_{o,0} =(\sigma_{x_{o,0}}, \sigma_{y_{o,0}}, \sigma_{\theta_{o,0}})$, which is assumed to grow linearly over the prediction horizon with an additive diagonal term $Q \in \mathbb{R}^{3}$, i.e., $\Sigma_{o, n|k} = \Sigma_{o,0} + (n-k)Q$, where $n$ is the prediction step. This models the effect that predictions further in the future are generally more uncertain. 
Note that our model does not consider the road layout and traffic rules. The motion dynamics (\ref{eq:sys}) are given by a discrete-time unicycle model, where the inputs are the velocity $v_e$ and turn rate $\omega_e$ and the outputs are configurations $\mathbf{y}_e$. All parameters are given in the Appendix \ref{appendixD}.
\begin{figure*}[t]
    \centering
    \begin{subfigure}[b]{1\textwidth}
        \includegraphics{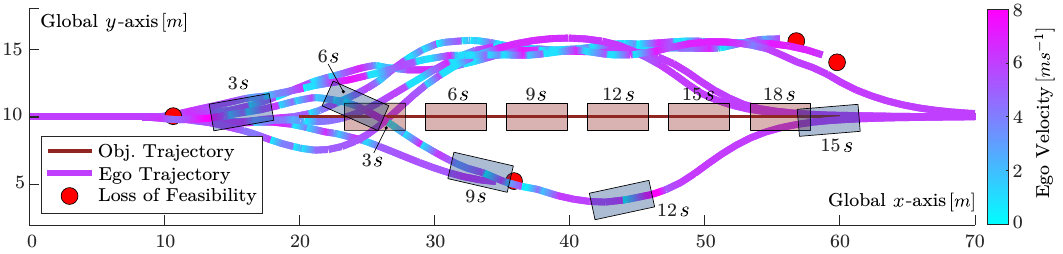}  
        \caption{Eight simulated overtaking maneuvers were performed using MCS estimation of the POC. The uncertainty is identical in each simulation. Note that the SMPC planner lost feasibility in four simulations, denoted by the red dot. Also, the ego vehicle's shape is depicted only for one simulation to preserve clarity.}
        \label{fig_SMPC_MCS}
    \end{subfigure}\hfill
    \begin{subfigure}[b]{1\textwidth}
        \includegraphics{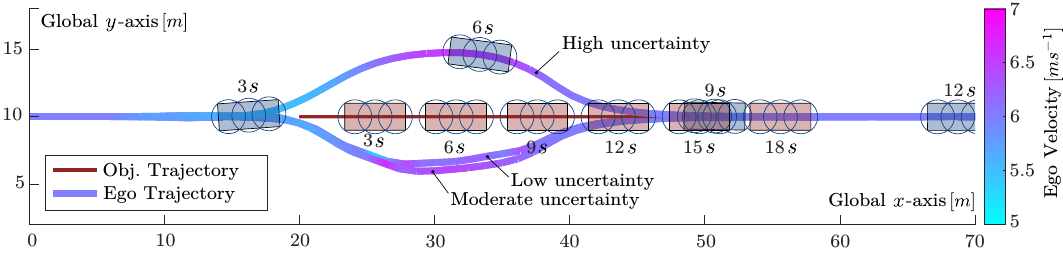}  
        \caption{Simulation of the overtaking scenario where the POC is estimated by the multi-circular approach for three different levels of uncertainties. Note that the scaling of the velocity bar differs from Figure \ref{fig_SMPC_MCS}. Also, the ego vehicle's shape is depicted only for the \textit{high uncertainty} case to preserve clarity.}
        \label{fig_SMPC_AIN}
    \end{subfigure}
    \caption{SMPC with estimating the POC by MCS and by the multi-circular approach.}
\end{figure*}
\subsubsection{Monte Carlo Sampling}

While it is suspected that estimating the POC with MCS is challenging due to the aforementioned concerns regarding POC fluctuations, we still simulate the overtaking scenario with MCS to confirm the practical relevance of the theoretical concerns. We used MCS with $10^3$ samples, as in that case the computational effort is only roughly a factor of two slower to the three-circle approximation (see right figure in Figure \ref{fig_error_runtime}). We simulated the overtaking scenario eight times, with identical scenario parameters and the uncertainty set to a moderate level (see Appendix \ref{appendixD}). Figure \ref{fig_SMPC_MCS} shows that each resultant ego trajectory is distinct and that four trajectories lead to a loss of the recursive feasibility. Also, the planner is varying the input signals significantly, which can be seen by the rapidly changing velocities and heading angles.

\subsubsection{Multi-Circle Approximation}
For the proposed multi-circle POC estimation, we, contrasting to Figure \ref{fig_SMPC_MCS}, chose to simulate three different levels of uncertainty because, as expected, for the same parameters and uncertainty, the planner always generates the exact same trajectory. The results are given in Figure \ref{fig_SMPC_AIN} for the three-circle case. The results show smooth overtaking maneuvers without the loss of recursive feasibility, strongly contrasting to Figure \ref{fig_SMPC_MCS}. Lastly, increased uncertainty leads the controller to generate more conservative trajectories, i.e., keeping more distance to the object vehicle, which is generally desired. 

\section{Discussion of the Results}\label{sec_discission}
The theoretical results predict that the POC is over-approximated when both actor's shapes are covered with overlapping circles, given certain geometric conditions. The simulation study indeed confirms this in all tested scenarios. Rather interesting, however, is the finding that the over-approximation error depends on the scenario parameters, e.g., how both actors approach each other. Further, it appeared that approximations with more than three circles each for the chosen vehicle sizes do not lead to further gains in estimation accuracy. At the same time, the authors argue that the over-approximation error is not too conservative for any motion planning application as the error for more than two circles is low and predetermined by the geometry and uncertainty.\\
Another aspect addressed in this article are the drawbacks of MCS as the state-of-the-art method for POC estimation. These drawbacks are two-fold: MCS is computationally intensive and it is not guaranteed to always over-approximate the POC for a finite amount of samples. That is, it is undesirable to under-approximate the POC for safety-critical applications, which MCS is prone to do. 
Regarding computational efficiency, we found that our proposed algorithm (see Algorithm \ref{alg:poc_overview} - \ref{alg:SortInts} is highly efficient in comparison to our MCS implementation. This is due to the algorithm structure that leverages vectorization and symbolic pre-computation; so only a symbolic expression needs to be evaluated at run-time. Further, for a set numerical integration accuracy and geometry, the number of circles only affects the initialization time, which in practice is irrelevant as a pre-computed symbolic POC expression only requires very little memory to store. The main disadvantages of MCS are highlighted in Figure \ref{fig_accuracy}, as its mean value may be accurate, though the fluctuations for limited numbers of samples have direct impact for motion planning, see Figure \ref{fig_SMPC_MCS}. While half of the trajectories led to a loss of feasibility, the trajectories themselves were rather erratic, and each was distinctly different. In contrast, the proposed method yields smooth trajectories (see Figure \ref{fig_SMPC_AIN}) and no loss of feasibility for any trajectory. \\
Further, we highlight the ease of working with the proposed method, as the reproducable smooth trajectories and low computational effort greatly help during the parametrization of the SMPC. This is also expected to affect AV safety and development, as potential bugs and errors can consistently be reproduced because the planner behaves deterministically between repetitive simulations for the same parameters. Based on our findings, we argue that our approach is a superior POC estimation method compared to MCS for SMPC applications.

\section{Conclusions and Future Work}\label{sec_conclusions_and_future_work}
In this paper, we present a method to over-approximate the probability of collision (POC) tailored for optimization-based motion planning algorithms. We thereby address the shortcomings of other methods, e.g., Monte Carlo sampling (MCS), where the computational efficiency, the lack of guarantees not to under-approximate the POC, and fluctuating POC estimation results leading to inconsistent motion plans are known problems.\\
The proposed method utilizes a multi-circular shape approximation for actors where one actor's position and heading angle are uncertain. We calculate the POC for the multi-circular shape approximation and show that it is an over-approximation of the POC of two rectangles colliding. We propose a POC algorithm for Gaussian uncertainties to emphasize the practical applicability of our theoretical results. Here, our selected order of integration allows us to reduce a three-dimensional integral to a two-dimensional integral, improving computational efficiency. Our POC algorithm is demonstrated and compared against MCS, where we confirm the obtained theoretical results and find that three-circle approximations are sufficiently accurate for motion planning applications while being computationally more than twenty times faster than our MCS implementation with $10^4$ samples. Lastly, we compare our method against MCS with a stochastic model predictive controller in an overtaking scenario. We observe that MCS is insufficient as it often renders the controller infeasible, while generating erratic, and none-reproducible trajectories caused by the inherent POC result fluctuations. Contrarily, our method generates smooth, reproducible trajectories and does not suffer from infeasibility issues. Further, our method leads the controller to deal with higher uncertainty by displaying a more conservative driving behavior. We conclude that our method is superior to MCS for SMPC applications, as it computes faster and generates reproducible, smooth trajectories.\\
For future work, computing the POC from different types of probability density functions that accommodate more complex object behaviors is of interest, as this article only in detail derives the POC for Gaussian distributed uncertainty. 

\bibliography{lib}
\bibliographystyle{ieeetr}
\appendices 
\section{Derivation of the Heading Angle Bounding}\label{appendixA} Consider the rectangle of sides $L_o, R$, and $\rho'$ in Figure \ref{fig_offset_circle}(c). Denote the inner angle between the sides $L_o$ and $\rho'$ by $\theta'$. With the law of cosine it holds that 
\begin{equation}\label{app_1}
    \frac{L_o^2 + \rho'^2 - R^2}{2L_o \rho'} = \cos{(\theta')}.
\end{equation}
Note that the left-hand side of (\ref{app_1}) satisfies the triangle inequality, i.e., $ \frac{L_o^2 + \rho'^2 - R^2}{2L_o \rho'} \in [0, 1]$ for $\rho' \in [0, L_o + R]$, restricting $\theta'_o$ to positive real values.
The lower-bound $\underline{\theta}$ and upper-bound $\overline{\theta}$, characterizing the first and last contact point between both circles, are given by 
\begin{equation} \label{app_2}
    \underline{\theta} = \phi' + \pi -\theta', \; \text{and } \overline{\theta} = \phi' + \pi + \theta'
\end{equation} 
(see Figure \ref{fig_offset_circle}(c)). By rearranging and substituting (\ref{app_2}) into (\ref{app_1}) the expression of the heading angle bounds (\ref{eq_circle_v_offset_circle_heading_bounds}) is retrieved. 
\section{POC Simulation}\label{appendixB}
The following parameters are used for the simulations performed in Section \ref{sec_poc_estimation}:
\textit{uncertainty}: $\gamma = 1, d_0 = 1, \Sigma_{max} =Diag(1, 1, 1)$. \textit{Vehicle dimensions (both vehicles)}: $w = 2, l = 4.5$ Scenario \textit{intersection crash}: \textit{Ego vehicle}: initial configuration $\mathbf{y}_{e,0} = (0, 4, 0)^T$, velocity $v_e = 1$ turn rate $\omega_e = 0$. \textit{Object}: initial configuration $\mathbf{y}_{o,0} = (4, 0, \pi/2)^T$, velocity $v_o = 1$ turn rate $\omega_o = 0$. Scenario \textit{intersection crash}: \textit{Ego vehicle}: initial configuration $\mathbf{y}_{e,0} = (0, 4, 0)^T$, velocity $v_e = 1$ turn rate $\omega_e = 0$. \textit{Object}: initial configuration $\mathbf{y}_{o,0} = (6, 0, \pi/2)^T$, velocity $v_o = 1.5$ turn rate $\omega_o = 0$. Scenario \textit{oncoming pass}: \textit{Ego vehicle}: initial configuration $\mathbf{y}_{e,0} = (0, 0, 0)^T$, velocity $v_e = 1$ turn rate $\omega_e = 0$. \textit{Object}: initial configuration $\mathbf{y}_{o,0} = (8, 3.5, \pi)^T$, velocity $v_o = 1$ turn rate $\omega_o = 0$. 
\section{POC Accuracy}\label{appendixC}
The following parameters are used for the simulations performed in Section \ref{sec:poc_accuracy}:
Ego configuration $\mathbf{y}_{e} = (0, 0, 0)^T$, Object configuration $\mathbf{y}_{o} = (2.5, 2.5, 0)^T$, vehicle dimensions (both vehicles): $w = 2, l = 4.5$,
low uncertainty: $Diag(0.5, 0.5, 0.5)$, moderate uncertainty: $Diag(1.5, 1.5, 1.5)$, high uncertainty: $Diag(2.5, 2.5, 2.5)$.
\section{Parametric settings SMPC}\label{appendixD}
The following parameters are used for the simulations performed in Section \ref{sec_results_SMPC}:
\textit{Vehicle dimensions (both vehicles)}: $w = 2, l = 4.5$, \textit{Ego vehicle}: initial configuration $\mathbf{y}_{e,0} = (0, 10, 0)^T$, input constraints $\mathcal{V}_e = [0, 10], \mathcal{U}_e = [-1, 1]$, weightings in cost function: $W = Diag(1, 1, 10, 10)$, sample time $0.2$, prediction horizon $N_P = 10$, MCS samples $10^3$, POC tolerance $\epsilon = 0.2$, low uncertainty: $\Sigma_{o,0} = Diag(0.1, 0.1, 0.1), Q = Diag(0.01, 0.01, 0.01)$, moderate uncertainty: $\Sigma_{o,0} = Diag(0.1, 0.1, 0.1), Q = Diag(0.3, 0.3, 0.3)$, high uncertainty: $\Sigma_{o,0} = Diag(0.5, 0.5, 0.5), Q = Diag(0.5, 0.5, 0.5)$. \textit{Object vehicle}: initial configuration $\mathbf{y}_{e,0} = (20, 10, 0)^T$, velocity $v_o = 2$, turn rate $\omega_o = 0$. 
\vspace{\negSpace}
\begin{IEEEbiography}[{\includegraphics[width=1in,height=1.25in,clip,keepaspectratio]{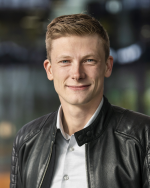}}]{Leon Tolksdorf} received the B.Eng. degree in automotive engineering and management and the M.Sc. degree in automotive engineering from the Munich University of Applied Sciences, in 2017 and 2019, respectively. He is currently pursuing a Ph.D. degree with the Dynamics and Control Group at the Eindhoven University of Technology, Eindhoven The Netherlands. Since 2020 he is a research assistant at the CARISSMA Institute of Safety in Future Mobility in Ingolstadt, Germany. His main research interests are stochastic optimization for motion planning, computation of stochastic risk measures, and risk-aware automated vehicle behavior generation in single and multi-agent control. 
\end{IEEEbiography}
\vspace{\negSpace}
\begin{IEEEbiography}[{\includegraphics[width=1in,height=1.25in,clip,keepaspectratio]{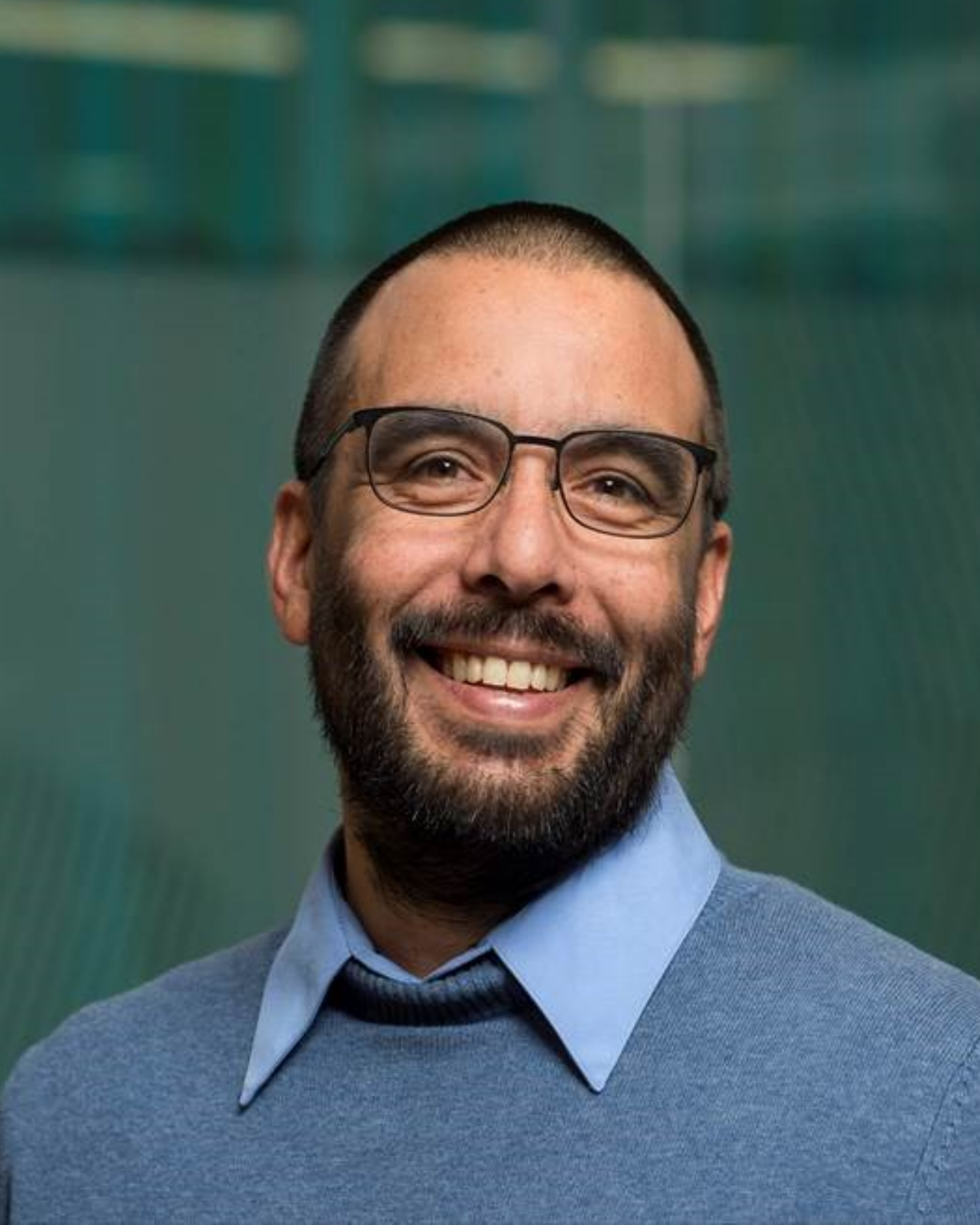}}]{Arturo Tejada} holds a BS degree in Electrical Engineering from the Pontificia Universidad Católica del Perú and MSc. and Ph. D. degrees in Electrical Engineering from Old Dominion University in Norfolk, Virginia (2006), where he specialized in hybrid system theory. He is the author of over 50 scientific publications in top peer-reviewed conferences and journals. He is a Part-time Assistant Professor of Safe Autonomous and Cooperative Vehicles at the Mechanical Engineering Department at Eindhoven University of Technology (TU/e) and a Senior Scientist at the Integrated Vehicle Safety Department of TNO in The Netherlands. His work focuses on "teaching" self-driving vehicles how to interact with human drivers safely and socially. This is done by developing reference models of human driving from which design requirements for automated driving functions can be extracted and demonstrated. His work sits at the intersection of human factors, artificial intelligence, and motion control of automated vehicles.
\end{IEEEbiography}
\vspace{\negSpace}
\begin{IEEEbiography}[{\includegraphics[width=1in,height=1.25in,clip,keepaspectratio]{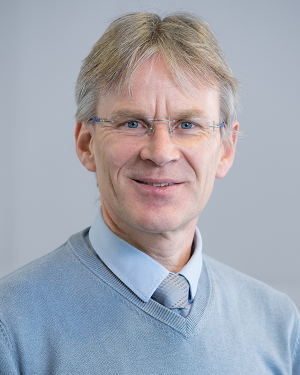}}]{Christian Birkner} holds a distinguished academic record and professional expertise in Mechanical Engineering. He acquired both his Dr.-Ing.-degree in Mechanical Engineering from the Technical University of Kaiserslautern in 1994. Commencing his professional career in 1995, Birkner accumulated a significant breadth of industrial experience through his work in powertrain design and control at SiemensVDO, IAV, and MAHLE until 2017. His work involved substantial contributions to the advancement of powertrain technology, solidifying his credentials as a mechanical engineer in the industry. In 2017, he transitioned into academia as a Professor and Chair of the Test Methods for Vehicle Safety, Vehicle Systems, and Control department at THI. His work in this position involves fostering a research environment focused on integrated vehicle safety.
\end{IEEEbiography}
\vspace{\negSpace}
\begin{IEEEbiography}[{\includegraphics[width=1in,height=1.25in,clip,keepaspectratio]{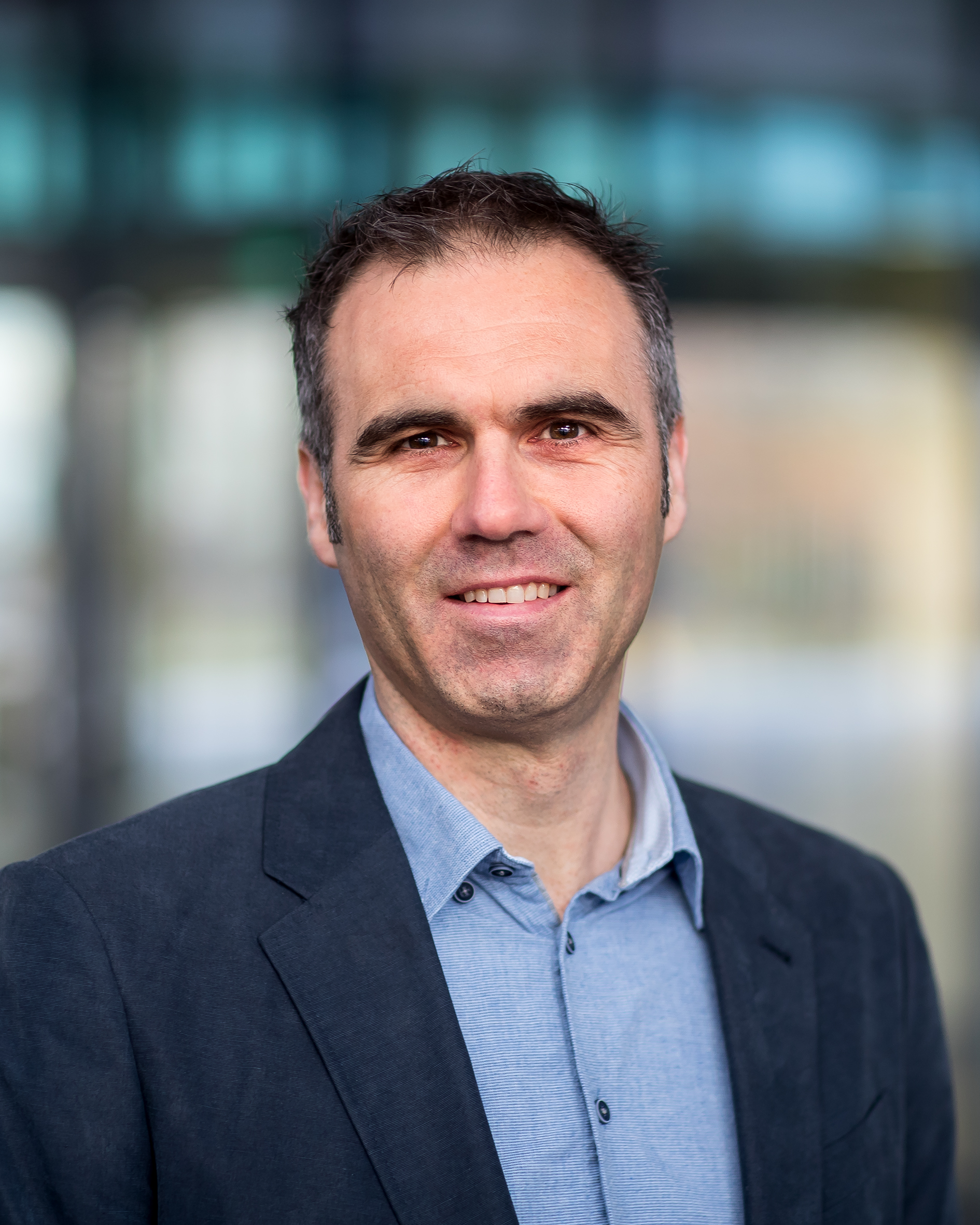}}]{Nathan van de Wouw} obtained his M.Sc.-degree (with honours) and
Ph.D.-degree in Mechanical Engineering from the Eindhoven
University of Technology, the Netherlands, in 1994 and 1999,
respectively. He currently holds a full professor position at the Mechanical Engineering Department of the Eindhoven University of Technology, the Netherlands. He has been
working at Philips Applied Technologies, The Netherlands, in
2000 and at the Netherlands Organisation for Applied Scientific
Research, The Netherlands, in 2001. He has been a visiting
professor at the University of California Santa Barbara,
U.S.A., in 2006/2007, at the University of Melbourne,
Australia, in 2009/2010 and at the University of Minnesota,
U.S.A., in 2012 and 2013. He has held a (part-time) full professor position at the Delft University of Technology, the Netherlands, from 2015-2019.
He has also held an adjunct full professor position at the University of Minnesota, U.S.A, from 2014-2021.
He has published the books 'Uniform Output Regulation of
Nonlinear Systems: A convergent Dynamics Approach' with A.V.
Pavlov and H. Nijmeijer (Birkhauser, 2005) and `Stability and
Convergence of Mechanical Systems with Unilateral Constraints'
with R.I. Leine (Springer-Verlag, 2008). In 2015, he received the IEEE Control Systems Technology Award "For the development and application of variable-gain control techniques for high-performance motion systems". He is an IEEE Fellow for his contributions to hybrid, data-based and networked control.

\end{IEEEbiography}
\end{document}